\documentclass[a4paper, 11pt]{article} 

\usepackage{latexsym,mathrsfs}
\usepackage{amsmath,amssymb}
\usepackage{amsthm,enumerate,verbatim}
\usepackage{amsfonts}
\usepackage{graphicx}
\usepackage{algorithm}
\usepackage{algorithmic}
\usepackage{url}
\usepackage{wrapfig}

\newtheorem{lemma}{Lemma}

\newtheorem{theorem}{Theorem}

\newtheorem{assumption}{Assumption}
\newtheorem{remark}{Remark}

\DeclareMathOperator{\argmax}{argmax} 
\DeclareMathOperator{\argmin}{argmin} 
\DeclareMathOperator{\rank}{rank}

\title{Fast and Robust Recursive Algorithms \\ for Separable Nonnegative Matrix Factorization\thanks{This work was supported in part by a grant from the U.S. Air Force Office of Scientific Research and a Discovery Grant from the Natural Science and Engineering Research Council (Canada).}}

\date{}

\author{Nicolas Gillis \\ 
Department of Mathematics and Operational Research  \\ 
       Facult\'e Polytechnique, Universit\'e de Mons \\
       Rue de Houdain 9, 7000 Mons, Belgium \\ 
Email: nicolas.gillis@umons.ac.be 
 \and 
 Stephen A. Vavasis \\ 
Department of Combinatorics and Optimization \\ 
University of Waterloo \\
Waterloo, Ontario N2L 3G1, Canada \\ 
Email: vavasis@math.uwaterloo.ca}

\begin{document}

\maketitle

\begin{abstract} 
 In this paper, we study the nonnegative matrix factorization problem under the separability assumption (that is, there exists a cone spanned by a small subset of the columns of the input nonnegative data matrix containing all columns), which is equivalent to the hyperspectral unmixing problem under the linear mixing model and the pure-pixel assumption. We present a family of fast recursive algorithms, and prove they are robust under any small perturbations of the input data matrix. This family generalizes several existing hyperspectral unmixing algorithms hence provide for the first time a theoretical justification of their better practical performances.   
\end{abstract} 


\section{Introduction}

%\begin{definition}[Separability] 

A hyperspectral image consists of a set of images taken at different wavelengths. It is acquired by measuring the spectral signature of each pixel present in the scene, that is, by measuring the reflectance (the fraction of the incident electromagnetic power that is reflected by a surface at a given wavelength) of each pixel at different wavelengths. 
One of the most important tasks in hyperspectral imaging is called unmixing. It requires the identification of the constitutive materials present in the image and estimation of their abundances in each pixel.  
The most widely used model is the \emph{linear mixing model}: the spectral signature of each pixel results from the additive linear combination of the spectral signatures of the constitutive materials, called endmembers,  where the weights of the linear combination correspond to the abundances of the different endmembers in that pixel. 

More formally, let the $m$-by-$n$ matrix $M$ correspond to a hyperspectral image with $m$ spectral bands and $n$ pixels, and where each entry $M_{ij} \geq 0$ of matrix $M$ is equal to the reflectance of the $j$th pixel at the $i$th wavelength. Hence, each column $m_j$ of $M$ corresponds to the spectral signature of a given pixel. 
Assuming the image contains $r$ constitutive materials whose spectral signatures are given by the vectors $w_k \in \mathbb{R}^m_+$ $1 \leq k \leq r$, we have, in the noiseless case,  
\[
m_j \; = \; \sum_{k=1}^r w_k \, h_{kj}, \quad \text{ for } j=1,2,\dots,n, 
\]
where $h_{kj} \geq 0$ is the abundance of the $k$th endmember in the $j$th pixel, with $\sum_{k=1}^r h_{kj} = 1$ $\forall j$. 
Defining the $m$-by-$r$ matrix $W = [w_1 \, w_2 \, \dots w_k] \geq 0$ and the $r$-by-$n$ matrix $H$ with $H_{kj} = h_{kj}$ $\forall j,k$, the equation above can be equivalently written as $M = WH$ where $M$, $W$ and $H$ are nonnegative matrices.  
 Given the nonnegative matrix $M$, hyperspectral unmixing amounts to recovery of the endmember matrix $W$ and the abundance matrix $H$. 
 This inverse problem corresponds to the nonnegative matrix factorization problem (NMF), which is a difficult~\cite{V09} and highly ill-posed problem~\cite{G12}. 
 
 However, if we assume that, for each constitutive material, there exists at least one pixel containing only that material (a `pure' pixel), then the unmixing problem can be solved in polynomial time: it simply reduces to identifying the vertices of the convex hull of a set of points. 
 This assumption, referred to as the \emph{pure-pixel assumption}~\cite{C94}, is essentially equivalent to the separability assumption~\cite{DS03}: a nonnegative matrix $M$ is called separable if it can be written as $M = WH$ where 
 %$W$ and $H$ are nonnegative and 
 each column of $W$ is equal, up to a scaling factor, to a column of $M$. In other words, there exists a cone spanned by a small subset of the columns of $M$ containing all columns (see Section~\ref{secass} for more details). 
 It is worth noting that this assumption also makes sense for other applications. For example, in text mining, each entry $M_{ij}$ of matrix $M$ indicates the `importance' of word $i$ in document $j$ (e.g., the number of appearances of word $i$ in text $j$). The factors $(W,H)$ can then be interpreted as follows: the  columns of $W$ represent the topics (i.e., bags of words) while the columns of $H$ link the documents to these topics. Therefore, 
\begin{itemize}
\item Separability of $M$ (that is, each column of $W$ appears as a column of $M$) requires that, for each topic, there exists at least one document discussing only that topic (a `pure' document). 
\item Separability of $M^T$ (that is, each row of $H$ appears as a row of $M$) requires that, for each topic, there exists at least one word used only by that topic (a `pure' word). 
 \end{itemize}
 These assumptions often make sense in practice and are actually part of several existing document generative models, see \cite{AGKM11, AGM12} and the references therein. 

\subsection{Previous Work} \label{pwork}
 
  We focus in this paper on hyperspectral unmixing algorithms under the linear mixing model and the pure-pixel assumption, or, equivalently, to nonnegative matrix factorization algorithms under the separability assumption. 
  Many algorithms handling this situation have been developed by the  remote sensing community, see \cite{BP12} for a comprehensive overview of recent hyperspectral unmixing algorithms.   Essentially, these algorithms amount to identifying the vertices of the convex hull of the (normalized) columns of $M$, or, equivalently, the extreme rays of the convex cone generated by the columns of $M$.  However, as far as we know, none of these algorithms have been proved to work when the input data matrix $M$ is only approximately separable (that is, the original separable matrix is perturbed with some noise), and many algorithms are therefore not robust to noise. However, there exists a few recent notable exceptions: 
 \begin{itemize}

 \item  Arora et al.\@ \cite[Section 5]{AGKM11} proposed a method which requires the resolution of $n$ linear programs in $\mathcal{O}(n)$ variables ($n$ is the number of columns of the input matrix), and is  therefore not suited to dealing with large-scale real-world problems. In particular, in hyperspectral imaging, $n$ corresponds to the number of pixels in the image and is of the order of $10^6$. Moreover, it needs several parameters to be estimated a priori (the noise level, and a function of the columns of $W$; see Section~\ref{compa}).

 \item Esser et al.\@ \cite{EMO12} proposed a convex model with $n^2$ variables (see also \cite{ESV1294} where a similar approach is presented), which is computationally expensive. In order to deal with a large-scale real-world hyperspectral unmixing problem, the authors had to use a preprocessing, namely $k$-means, to select a subset of the columns in order to reduce the dimension $n$ of the input matrix. Their technique also requires a parameter to be chosen in advance (either the noise level, or a penalty parameter balancing the importance between the approximation error and the number of endmembers to be extracted), only applies to a restricted noise model, and cannot deal with repeated columns of $W$ in the data set (i.e., repeated endmembers).

 \item  Bittorf et al.\@ \cite{BRRT12} proposed a method based on the resolution of a single convex optimization problem in $n^2$ variables (cf.\@ Section~\ref{compaB}). In order to deal with large-scale problems ($m~\sim~10^6$, $n~\sim~10^5$), a fast incremental gradient descent algorithm using a parallel architecture is implemented. However, the algorithm requires several parameters to be tuned, and the factorization rank has to be chosen a priori. Moreover, it would be impractical for huge-scale problems (for example for web-related applications where $n \sim 10^9$), and the speed of convergence could be an issue.

   \end{itemize}

   \subsection{Contribution and Outline of the Paper} 
   
  In this paper, we propose a new family of recursive algorithms for nonnegative matrix factorization under the separability assumption. They have the following features: 
  %. These algorithms have the following features 
  \begin{itemize} 
  \item They are robust to noise (Theorem~\ref{threc}).  
  
  \item They are very fast, running in approximately $6mnr$ floating point operations, while the memory requirement is low, as only one $m$-by-$n$ matrix has to be stored.  
  
  \item They are extremely simple to implement and would be easily parallelized.     
  
  \item They do not require any parameter to be chosen a priori, nor to be tuned. 
  
  \item The solution does not need to be recomputed from scratch when the factorization rank is modified, as the algorithms are recursive. 
  
  \item A simple post-processing strategy allows us to identify outliers (Section~\ref{outliers}).  
  
  \item Repeated endmembers are not an issue. 
  
  \item Even if the input data matrix $M$ is not approximately separable, they identify $r$ columns of $M$ whose convex hull has large volume (Section~\ref{l2norm}). 
  
   \end{itemize}
   To the best of our knowledge, no other algorithms share all these desirable properties. %As there is no free lunch, 
   The weak point of our approach is that the bound on the noise to guarantee recovery    is weaker than in~\cite{AGKM11, BRRT12}; see Section~\ref{compa}. 
   Also, we will need to assume that the matrix $W$ is full rank, which is not a necessary condition for the approaches above \cite{AGKM11, EMO12, BRRT12}. However, in practice, this condition is satisfied in most cases. At least, it is always assumed to hold in hyperspectral imaging and text mining applications, otherwise the abundance matrix $H$ is typically not uniquely determined; see Section~\ref{secass}. Moreover, in Section~\ref{compaB}, our approach will be shown to perform  in average better  than the one proposed in \cite{BRRT12} on several synthetic data sets.

 The paper is organized as follows. In Section~\ref{algo1sec}, we introduce our approach and derive an a priori bound on the noise to guarantee the recovery of the pure pixels. 
  In Section~\ref{outliers}, we propose a simple way to handle outliers. 
 In Section~\ref{choicef}, we show that this family of algorithms generalizes several hyperspectral unmixing algorithms, including the successive projection algorithm (SPA)~\cite{MC01},  the automatic target generation process (ATGP)~\cite{RC03}, the successive volume maximization algorithm (SVMAX)~\cite{CM11}, and the $p$-norm based pure pixel algorithm (TRI-P)~\cite{A11}. 
 Therefore, our analysis gives the first theoretical justification of the better performances of this family of algorithms compared to algorithms based on locating pure pixels using linear functions (such as the widely used PPI \cite{B94} and VCA \cite{ND05} algorithms) which are not robust to noise. This was, until now, only experimentally observed. Finally, we illustrate these theoretical results on several synthetic data sets in Section~\ref{ne}.  \\

\noindent \textbf{Notation.} Given a matrix $X$, $x_k$, $X_{:k}$ or $X(:,k)$  denotes its $k$th column, and $X_{ik}$, $x_{ik}$ or $x_k(i)$ the entry at position $(i,k)$ ($i$th entry of column $k$). For a vector $x$, $x_i$ or $x(i)$ denotes the $i$th entry of $x$. 
The unit simplex in dimension $n$ is denoted $\Delta^n = \{ x \in \mathbb{R}^n \ | \ x \geq 0, \sum_{i=1}^n x_i \leq 1\}$. 
%we will drop the index $n$ when the dimension is clear from the context. 
We use the MATLAB notation $[A, \; B] = \left( \begin{array}{cc} A & B \\ \end{array} \right)$ and $[A; B] = \binom{A}{B}$. 
%\left( \begin{array}{c} A \\ B \\ \end{array} \right)$.  
Given a matrix $W \in \mathbb{R}^{m \times r}$, $m \geq r$, we denote $\sigma_i(W)$ the singular values of $W$ in non-decreasing order, that is, 
\[
\sigma_1(W) \geq \sigma_2(W) \geq \dots \geq \sigma_r(W) \geq 0. 
\] 
The $m$-by-$n$ all-zero matrix is denoted $0_{m \times n}$ while the $n$-by-$n$ identity matrix is denoted $I_{n}$ (the subscripts $m$ and $n$ might be discarded if the dimension is clear from the context).

\section{Robust Recursive NMF Algorithm under Separability} \label{algo1sec}

In this section, we analyze a family of simple recursive algorithms for NMF under the separability assumption; see Algorithm~\ref{sepnmf}.  
\algsetup{indent=2em}
\begin{algorithm}[ht!]
\caption{Recursive algorithm for separable NMF \label{sepnmf}}
\begin{algorithmic}[1] 
\REQUIRE Separable matrix $M = WH \in \mathbb{R}^{m \times n}_+$ (see Assumption~\ref{asssep}),  the number $r$ of columns to be extracted, and a strongly convex function $f$ satisfying Assumption~\ref{fass1}. 
\ENSURE Set of indices $J$ such that $M(:,J) = W$ up to permutation; cf.\@ Theorems~\ref{th1}~and~\ref{threc}.
    \medskip 
%\STATE Remove the zero columns from $M$. 
%\STATE Set the $\ell_1$-norm of the columns of $M$ to one. 
\STATE Let $R = M$, %$I = \{1,2,\dots,n\}$, 
$J = \{\}$, $j=1$.  
\WHILE {$R \neq 0$ and $j \leq r$}   
\STATE $j^* = \argmax_j f(R_{:j})$.  $\dagger$
\STATE $u_j = {R_{:j^*}}$.  \vspace{0.1cm} 
%\IF{$j = 1$} 
%\STATE $R_{:i} \leftarrow R_{:i} - u_j$ for $i = 1, 2, \dots, n$. 
%\ELSE
\STATE $R \leftarrow \left(I-\frac{u_j u_j^T}{||u_j||_2^2}\right)R$. \vspace{0.1cm} 
%\ENDIF 
\STATE $J = J \cup \{j^*\}$. 
\STATE $j = j+1$.
\ENDWHILE
\end{algorithmic}
$\dagger$ In case of a tie, we pick an index $j$ such that the corresponding column of the original matrix $M$ maximizes $f$.  
\end{algorithm} 
Given an input data matrix $M$ and a function $f$, it works as follows: at each step, the column of $M$ maximizing the function $f$ is selected, and $M$ is updated by projecting each column onto the orthogonal complement of the selected column.   

\begin{remark}[Stopping criterion for Algorithm~\ref{sepnmf}] Instead of fixing a priori the number $r$ of columns of the input matrix to be extracted, it is also possible to stop the algorithm whenever the norm of the residual (or of the last extracted column) is smaller than some specified threshold. 
\end{remark} 

In Section~\ref{secass}, we discuss the assumptions on the input separable matrix $M = WH$ and the function $f$ that we will need in Section~\ref{noiseless} to prove that Algorithm~\ref{sepnmf} is guaranteed to recover columns of $M$ corresponding to columns of the matrix $W$. 
 Then, we analyze Algorithm~\ref{sepnmf} 
 in case some noise is added to the input separable matrix $M$, and show that, under these assumptions, it is robust under any small perturbations; see Section~\ref{noisysec}. 
Finally, we compare our results with the ones from \cite{AGKM11, BRRT12} in Section~\ref{compa}.

\subsection{Separability and Strong Convexity Assumptions} \label{secass}

In the remainder of the paper, we will assume that the original data matrix $M = WH$ is separable, that is, each column of $W$ appears as a column of $M$. 
%there exists a cone spanned by a small subset of the columns of $M$ that contains all columns of $M$. 
Recall that this condition is implied by the pure-pixel assumption in hyperspectral imaging; see Introduction. 
We will also assume that \textit{the matrix $W$ is full rank}. This is often implicitly assumed in practice otherwise the problem is in general ill-posed, because the matrix $H$ is then typically not uniquely determined; see, e.g., \cite{AGKM11, SX11}.

\begin{assumption} \label{asssep}
The separable matrix $M \in \mathbb{R}^{m \times n}$ can be written as 
%\[
 $M = W H = W [I_r, H']$,  
%\]
where $W \in \mathbb{R}^{m \times r}$ has rank $r$, $H \in \mathbb{R}^{r \times n}_+$, and the sum of the entries of each column of $H'$ is at most one, that is, $\sum_{k=1}^r H'_{kj} \leq 1$ $\forall j$, or, equivalently, $h'_j \in \Delta^r$ $\forall j$. 
\end{assumption} 
The assumption on matrix $H$ is made without loss of generality by 
\begin{enumerate}
\item[(1)] Permuting the columns of $M$ so that the first $r$ columns of $M$ correspond to the columns of $W$ (in the same order). 
%\item Assuming the columns of $W$ do not contain zero columns (otherwise they can be removed along with the corresponding rows of $H$). 
\item[(2)] Normalizing $M$ so that the entries of each of its columns sum to one (except for its zero columns).  In fact, we have that 
%Since the columns of $M$ are normalized to one, we can assume without loss of generality that the columns of $M$, $W$ and $H$ sum to one since 
\[
M = WH \iff M D_M^{-1} = W D_W^{-1} (D_W H D_M^{-1}),
\]
where 
\[
(D_X)_{ij} = \left\{ 
\begin{array}{cc}  
||X_{:j}||_1 & \text{ if $i = j$ and $X_{:j} \neq 0$,} \\ 
1 			     & \text{ if $i = j$ and $X_{:j} = 0$,} \\ 
0   & \text{ otherwise.} 
\end{array} \right.
\] 
By construction, the entries of each column of $M D_M^{-1}$ and $W D_W^{-1}$ sum to one (except for the zero columns of $M$), while the entries of each column of $(D_W H D_M^{-1})$ have to sum to one (except for ones corresponding to the zero columns of $M$ which are equal to zero) since $M = WH$.  %Hence each column of $M$ is in the convex hull of the columns of $W$, while each column of $W$ appears as a column of $M$.  
\end{enumerate}

In the hyperspectral imaging literature, the entries of each column of matrix $H$ are typically assumed to sum to one, hence Assumption~\ref{asssep} is slightly more general.   %(referred to as the sum-to-one constraint), 
This has several advantages: 
\begin{itemize}

\item It allows the image to contain `background' pixels with zero spectral signatures, which are present for example in hyperspectral images of objects in outer space (such as satellites).

\item It allows us to take into account different intensities of light among the pixels in the image, e.g., if there are some shadow parts in the scene or if the angle between the camera and the scene varies. Hence, although some pixels contain the same material(s) with the same abundance(s), their spectral signature could differ by a scaling factor. 

\item In the noisy case, it allows us to take into account endmembers with very small spectral signature %or with very low abundances 
as noise, although it is not clear whether relaxing the sum-to-one constraint is the best approach~\cite{BP12}. 
\end{itemize} 

\begin{remark}
Our assumptions actually do not require the matrix $M$ to be nonnegative, as $W$ can be any full-rank matrix. In fact, after the first step of Algorithm~\ref{sepnmf}, the residual matrix will typically contain negative entries. 
\end{remark}

We will also need to assume that the function $f$ in Algorithm~\ref{sepnmf} satisfies the following conditions. 
\begin{assumption} \label{fass1} 
The function $f:\mathbb{R}^m \to \mathbb{R}_+$ is strongly convex with parameter $\mu > 0$, its gradient is Lipschitz continuous with constant $L$, and its global minimizer is the all-zero vector with $f(0) = 0$. 
\end{assumption}

Notice that, for any strongly convex function $g$ whose gradient is Lipschitz continuous and whose global minimizer is $\bar{x}$, one can construct the function $f(x) = g(\bar{x}+x) - g(\bar{x})$ satisfying Assumption~\ref{fass1}. In fact, $f(0) = 0$ while $f(x) \geq 0$ for any $x$ since $g(\bar{x}+x) \geq g(\bar{x})$ for any $x$.  Recall that (see, e.g., \cite{UL01}) a function is strongly convex with parameter $\mu$ if and only if it is convex and for any $x, y \in \text{dom}(f)$ 
\[
f(\delta x + (1-\delta) y) 
\; \leq \;  
\delta f(x) + (1-\delta) f(y) - \frac{\mu}{2} \delta   (1-\delta) ||x-y||_2^2, 
\]
for any $\delta \in [0,1]$. Moreover, its gradient is Lipschitz continuous with constant $L$ if and only if for any $x, y \in \text{dom}(f)$ 
\[
||\nabla f (x) - \nabla f (y) ||_2 \leq L ||x-y||_2. 
\]
Convex analysis also tells us that if $f$ satisfies Assumption~\ref{fass1} then, for any $x, y$, 
\[
f(x) + \nabla f(x)^T (y-x) + \frac{\mu}{2} ||x-y||_2^2 
\leq
f(y)
\] 
and 
\[ 
f(y) 
\leq 
f(x) + \nabla f(x)^T (y-x) + \frac{L}{2} ||x-y||_2^2. 
\]
In particular, taking $x=0$, we have, for any $y \in \mathbb{R}^m$, 
\begin{equation} \label{normconv}
 \frac{\mu}{2} ||y||_2^2 
 \quad \leq  \quad 
f(y)
 \quad \leq  \quad 
\frac{L}{2} ||y||_2^2, 
\end{equation}
since $f(0) = 0$ and $\nabla f(0) = 0$ (because zero is the global minimizer of~$f$). \\ 

The most obvious choice for $f$ satisfying Assumption~\ref{fass1} is $f(x)=||x||^2_2$; we return to this matter in Section~\ref{l2norm}.

\subsection{Noiseless Case} \label{noiseless}

We now prove that, under Assumption~\ref{asssep} and \ref{fass1}, Algorithm~\ref{sepnmf} recovers a set of indices corresponding to the columns of $W$. 

\begin{lemma} \label{lemma1}
Let $Y = [W, 0_{m \times (r-k)}] = [w_1 \, w_2 \, \dots \, w_k \, 0_{m \times (r-k)}] \in \mathbb{R}^{m \times r}$ with $w_i \neq 0$ $\forall i$ and  $w_i \neq w_j$ $\forall i \neq j$, and let  $f:\mathbb{R}^m \to \mathbb{R}_+$ be a strongly convex function with $f(0) = 0$. Then 
\[
f(Yh) < \max_i f(w_i) \quad \text{ for all } h \in \Delta^r \text{ such that } h \neq e_j \forall j, 
\]
where $e_j$ is the $j$th column of the identity matrix. 
\end{lemma} 
\begin{proof}  
By assumption on $f$, we have $f(w) > 0$ for any $w \neq 0$; see Equation~\eqref{normconv}. 
%since $w_i \neq 0$ for all $i$, and 
%\[
%0 \leq f(t w_i + (1-t) 0 ) < t f( w_i ) + (1-t) f(0) = t f( w_i ), \quad \text{ for any } t \in (0,1). 
%\] 
Hence, if $Yh = 0$, we have the result since $f(Yh) = 0 <  f(w_i)$ for all $i$.  Otherwise $Yh = \sum_{i=1}^k w_i h_i$ where $h_i \neq 0$ for at least one $1\leq i \leq k$ so that 
\begin{align*}
f(Yh) & = f\left(\sum_{i=1}^k h_i w_i + \left(1-\sum_{i=1}^k h_i\right) 0
\right) \\ 
& < \sum_{i=1}^k h_i f\left( w_i \right) \leq \max_i f(w_i). 
\end{align*} 
The first inequality is strict since $h \neq e_j \forall j$ and $h_i \neq 0$ for at least one $1\leq i \leq k$, and the second follows from the fact that $\sum_{i=1}^k h_i \leq \sum_{i=1}^r h_i \leq 1$. 
\end{proof}
 
 %Therefore, we will assume that the columns of $M$ belong to the convex hull of the columns of $W \cup \{0\}$, and the vertices appear as the first $r$ columns of $M$. 

\begin{theorem} \label{th1}
Let the matrix $M = WH$ satisfy Assumption~\ref{asssep} and the function $f$ satisfy Assumption~\ref{fass1}. Then Algorithm~\ref{sepnmf} recovers a set of indices $J$ such that $M(:,J)=W$ up to permutation. 
%the $r$ vertices of $\conv(\theta(M))$ in $r$ steps. 
\end{theorem}
\begin{proof}
Let us prove the result by induction. 
 
\textit{First step.} Lemma~\ref{lemma1} applies since $f$ satisfies Assumption~\ref{fass1} while $W$ is full rank. Therefore, the first step of Algorithm~\ref{sepnmf} extracts one of the columns of $W$.  Assume without loss of generality the last column $w_r$ of $W$ is extracted, then the first  residual has the form  
 \begin{align*} 
 R^{(1)}   %& = \left(I - \frac{{w}_{r} {w}_{r}^T}{||{w}_{r}||_2^2}\right) M   \\
   & = \left(I - \frac{{w}_{r} {w}_{r}^T}{||{w}_{r}||_2^2}\right) W H   
   = [W' \, 0_{m \times 1}] H = W^{(1)} H, 
 \end{align*}
 i.e., the matrix $R^{(1)}$ is obtained by projecting the columns of $M$ onto the orthogonal complement of $w_{r}$. We observe that $W^{(1)}$ satisfies the conditions of Lemma~\ref{lemma1} as well because $W'$ is full rank since $W$ is. This implies, by Lemma~\ref{lemma1}, that the second step of Algorithm~\ref{sepnmf} extracts one of the columns of~$W'$. 
  
\textit{Induction step.} Assume that after $k$ steps the residual has the form $R^{(k)} = [W^* \; \mathbf{0}_{m \times k}] H$ with $W^*$ full rank. Then, by Lemma~\ref{lemma1}, the next extracted index will correspond to one of the columns of $W^*$ (say, without loss of generality, the last one) and the next residual will have the form 
$R  %= \left(I - \frac{{w^*}_{r-k} {w^*}_{r-k}^T}{||{w^*}_{r-k}||_2^2}\right) W H 
= [W^{\dagger}, \mathbf{0}_{m \times (k+1)}] H$ where $W^{\dagger}$ full rank since $W^{*}$ is, and $H$ is unchanged. By induction, after $r$ steps, we have that the indices corresponding to the different columns of $W$ have been extracted and that the residual is equal to zero ($R = \mathbf{0}_{m \times r}H$). 
\end{proof}

\subsection{Adding Noise} \label{noisysec}

In this section, we analyze how perturbing the input data matrix affects the performances of Algorithm~\ref{sepnmf}. We are going to assume that the input perturbed matrix $M'$ can be written as $M' = M + N$ where $M$ is the noiseless original separable matrix satisfying Assumption~\ref{asssep}, and $N$ is the noise with $||n_i||_2 \leq \epsilon$ for all $i$ for some sufficiently small $\epsilon \geq 0$.

\subsubsection{Analysis of a Single Step of Algorithm~\ref{sepnmf}}

Given a matrix $W$, we introduce the following  notations: 
%\begin{equation} \nonumber 
%\omega(W) = \min \left\{ \frac{1}{\sqrt{2}} \min_{i\neq j} ||w_i-w_j||_2 , \min_{i} ||w_i||_2   \right\},  
%\end{equation}
$\gamma(W) = \min_{i\neq j} ||w_i-w_j||_2$, $\nu(W) =  \min_{i} ||w_i||_2$, $\omega(W) = \min \left\{\nu(W), \frac{1}{\sqrt{2}} \gamma(W)\right\}$,  and   $K(W) =  \max_{i} ||w_i||_2$.  
%Moreover, given $\bar{\epsilon} \leq \nu(W)$, we denote 
%\[
%\bar{\omega}(W) = \min (\omega(W), \nu(W) - \bar{\epsilon}) \geq \omega(W) - \bar{\epsilon}. 
%\]

\begin{lemma} \label{lemscf}
Let $Y = [W, Q]$ where $W \in \mathbb{R}^{m \times k}$ and $Q \in \mathbb{R}^{m \times (r-k)}$, and let $f$ satisfy Assumption~\ref{fass1}, with  strong convexity parameter $\mu$ and its gradient having Lipschitz constant~$L$.  If 
\[
 \nu(W) 
 > 2 \sqrt{\frac{L}{\mu}} \, K(Q) \,  ,  
\]
then, for any $0 \leq \delta \leq \frac{1}{2}$, 
\begin{equation} \label{simfs}
f^* \quad = \quad \max_{x \in \Delta^r} f(Yx) \; \text{ such that }  x_i \leq 1 - {\delta} \, \text{ for } 1 \leq i \leq k, 
\end{equation}
satisfies 
%\begin{align}
$f^* %& \leq \max_i f(w_i) -  \, \mu \,  (1-\delta) \, \delta \, \min\left( {\omega}(W),\frac{1}{\sqrt{2}}(\nu(W)-K(E))\right)^2 \nonumber \\
 \leq \max_i f(w_i) -  \frac{1}{2} \, \mu \,  (1-\delta) \, \delta \, \omega(W)^2$. %\label{ineq1}
%\end{align}
\end{lemma} 
\begin{proof} 
By strong convexity of $f$, the optimal solution $x^*$ of \eqref{simfs} is attained at a vertex of the feasible domain $\{ x \in \mathbb{R}^r \ | \ x_i \geq 0 \, \forall i, \sum_{i=1}^r x_i \leq 1, x_i \leq 1 - {\delta} \, 1 \leq i \leq k  \}$, that is, either
\begin{enumerate}
\item[(a)]  $x^* = 0$, 
\item[(b)] $x^* = e_i$ for $k+1 \leq i \leq r$, 
\item[(c)] $x^* = (1-\delta)e_j$ for $1 \leq j \leq k$, 
\item[(d)] $x^* = \delta e_i + (1-\delta)e_j$ for $1 \leq i,j \leq k$, or 
\item[(e)] $x^* = \delta e_i + (1-\delta)e_j$ for $k+1 \leq i \leq r$ and $1 \leq j \leq k$. 
\end{enumerate}
Before analyzing the different cases, let us provide a lower bound for $f^*$. Using  Equation~\eqref{normconv}, we have 
\[
f((1-\delta) w_i) %\geq \frac{\mu}{2} ||w_i||_2^2 
\geq
\frac{1}{2} \mu (1-\delta)^2 ||w_i||_2^2 . 
\]
Since $(1-\delta) w_i$ is a feasible solution and  $0 \leq \delta \leq \frac{1}{2}$, this implies $f^* \geq  \frac{\mu}{8} K(W)^2$. 
 Recall that since $f$ is strongly convex with parameter $\mu$, we have 
\[
 f\left(\delta y + (1-\delta) z\right) 
\leq \delta f(y) + (1-\delta) f(z) - \frac{1}{2} \mu \delta (1-\delta) ||y-z||_2^2. 
\]
Let us now analyze the different cases.  

\begin{enumerate}

\item[(a)] Clearly, $x^* \neq 0$ since $f(0) = 0$ and $f(y) > 0$ for all $y \neq 0$, cf.~Equation~\eqref{normconv}.

\item[(b)] $Yx^* = q_i$ for some $i$. Using Equation~\eqref{normconv}, we have 
\begin{align*}
f^*  = f(q_i) 
%& \leq \frac{L}{2} ||q_i||_2^2 \\
& \leq \frac{L}{2}  K(Q)^2 %\\ & 
< \frac{\mu}{8} \nu(W)^2 
  \leq \frac{\mu}{8} K(W)^2 \leq f^*, 
\end{align*}
since $\nu(W) > 2 \sqrt{\frac{L}{\mu}} K(Q)$, a contradiction.  
%. Hence the objective function values of these vertices are smaller than $f^*$ hence cannot be optimal. 

\item[(c)] $Yx^* = (1-\delta) w_i$ for some $i$ : 
\begin{equation} \label{eqabove}
%\frac{\mu}{8} K(W)^2 
%\leq
f^* 
\leq (1-\delta) f(w_i) - \frac{1}{2} \mu \delta (1-\delta) ||w_i||_2^2 . 
%\leq \max_i f(w_i) - \frac{1}{2} \mu \delta (1-\delta) \omega^2,  
\end{equation}
By strong convexity, we also have 
$f(w_i) \geq \frac{\mu}{2} ||w_i||_2^2 
\geq  
\frac{1}{2} \mu (1-\delta) ||w_i||_2^2$.  
%\] 
Plugging it in \eqref{eqabove} gives
\begin{align*} 
f^* 
& \leq f(w_i) - \delta f(w_i) - \frac{1}{2} \mu \delta (1-\delta) ||w_i||_2^2 \\
& \leq  f(w_i) - \mu \delta (1-\delta) ||w_i||_2^2 \\ 
& \leq \max_i f(w_i) - \mu \delta (1-\delta) \nu(W)^2 . 
\end{align*}

\item[(d)] $Yx^* = \delta w_i + (1-\delta) w_j$ for some $i \neq j$ :  
\begin{align*} 
f^* & \leq \delta f(w_i) + (1-\delta) f(w_j) - \frac{1}{2} \mu \delta (1-\delta) ||w_i-w_j||_2^2 \\
& \leq \max_i f(w_i) - \mu \delta (1-\delta) \left(\frac{1}{\sqrt{2}}\gamma(W)\right)^2. 
\end{align*}

\item[(e)] $Yx^* = \delta q_i + (1-\delta) w_j$ for some $i$, $j$. First, we have 
\begin{align*}
f^* & \leq \delta f(q_i) + (1-\delta) f(w_j) - \frac{1}{2} \mu \delta (1-\delta) ||q_i-w_j||_2^2 \\
%& \leq \delta f(e_i) + (1-\delta) f(w_j) - \frac{1}{2} \mu \delta (1-\delta) (||w_j||_2 - K(E) )^2 \\
& \leq f(w_j) + \delta f(q_i) - \delta f(w_j)  \\ 
& \quad \;  - \frac{1}{2} \mu \delta (1-\delta) (\nu(W)-K(Q) )^2 \, .  %\\
%& \leq \max_j f(w_j) - \frac{1}{2} \mu \delta (1-\delta) (\nu(W)-K(E) )^2 . 
\end{align*} 
In fact, $||q_i-w_j||_2 \geq ||w_j||_2 - ||q_i||_2 \geq \nu(W) - K(Q)$. Then, using 
%\geq \frac{1}{2} \nu(W) \geq 0$, 
%and $f(e_i) < f(w_j)$ for all $i$,$j$ since 
\[
f(q_i) \leq \frac{L}{2} ||q_i||_2^2 \leq \frac{L}{2} K(Q)^2 < \frac{1}{4} \frac{\mu}{2} \nu(W)^2 \leq \frac{1}{4} f(w_j), 
\]
 $\nu(W) - K(Q) > \left(1-\frac{1}{2} \sqrt{\frac{\mu}{L}}\right) \nu(W) \geq \frac{1}{2}\nu(W)$ and $f(w_j) \geq \frac{\mu}{2} (1-\delta) \nu(W)^2$, we obtain
\begin{align*}
f^* & < f(w_j) - \frac{3}{4} \delta f(w_j)  - \frac{1}{8} \mu \delta (1-\delta) \nu(W)^2 \\
 & \leq f(w_j) - \frac{1}{2} \mu \delta (1-\delta) \nu(W)^2 . 
\end{align*} 
\end{enumerate} 
%The last inequality \eqref{ineq1} follows from $\nu(W) \geq \omega(W)$ and  $K(E) < \frac{1}{2} \sqrt{\frac{\mu}{L}} \nu(W) \leq \frac{1}{2} \nu(W)$. 
  %and $\mu \le L$. 
\end{proof}

\begin{lemma} \label{fbound}
Let the function $f : \mathbb{R}^m \to \mathbb{R}_+$ satisfy Assumption~\ref{fass1}, with strong convexity parameter $\mu$ and its gradient having Lipschitz constant $L$.  
Then, for any $x, n \in \mathbb{R}^m$ and any $\epsilon, K \geq 0$ satisfying $||x||_2 \leq K$ and $||n||_2 \leq \epsilon \leq K$, we have 
\begin{equation} \label{ufu}
f(x+n) \leq f(x) + \left(\epsilon K L + \frac{L}{2} \epsilon^2\right) \leq f(x) + \frac{3}{2} \epsilon K L, \text{ and } 
\end{equation}
\begin{equation} \label{ufl}
f(x+n) \geq f(x) - \left(\epsilon  K L - \frac{\mu}{2} \epsilon^2 \right) \geq f(x) - \epsilon  K L. 
\end{equation} 
%\[
%u_f(K,\epsilon) = 2 \epsilon K L + \frac{L - \mu}{2} \epsilon^2 \leq \frac{5}{2} \epsilon K L.
%\]
\end{lemma}
\begin{proof} 
For the upper bound \eqref{ufu}, we use the fact that the gradient of $f$ is  Lipschitz continuous with constant $L$ 
\begin{align*}
f(x+n) 
& \leq f(x) + \nabla f(x)^T n + \frac{L}{2} ||n||_2^2 \\
& \leq f(x) + \epsilon K L + \frac{L}{2} \epsilon ^2  \leq f(x) + \frac{3}{2} \epsilon KL, 
\end{align*}
for any $||x||_2 \leq K$, $||n||_2 \leq \epsilon \leq K$. The second inequality follows from the fact that $\nabla f(0) = 0$ and by Lipschitz continuity of the gradient:  $||\nabla f(x) - 0||_2 \leq L ||x-0||_2 \leq L K$  for any $||x||_2 \leq K$.

For the lower bound \eqref{ufl}, we use  strong convexity 
\begin{align*}
f(x+n) 
& \geq f(x) + \nabla f(x)^T n + \frac{\mu}{2} ||n||_2^2 \\
& \geq f(x) - K L ||n||_2  + \frac{\mu}{2} ||n||_2^2 \\ 
& \geq f(x) -  \left(\epsilon  K L - \frac{\mu}{2} \epsilon^2\right),  
%\geq f(x) - \epsilon ||\nabla f(x)||_2 \geq f(x) - \epsilon KL, 
\end{align*}
for any $||x||_2 \leq K$, $||n||_2 \leq \epsilon \leq K$. % The second inequality follows from the fact that 0 is the global minimizer of $f$ hence $\nabla f(0) = 0$ so that $||\nabla f(x) - 0||_2 \leq L ||x-0||_2 \leq L K$  for any $||x||_2 \leq K$. 
The third inequality follows from the fact that
\[
\max_{0 \leq y \leq \epsilon} \left( y  K L - \frac{\mu}{2} y^2 \right) = \epsilon  K L - \frac{\mu}{2} \epsilon^2 , 
\]
since $K \geq \epsilon$ and $L \geq \mu$.  
%Finally, 
%\[
%u_f(K,\epsilon) = 2 \epsilon K L + \frac{L - \mu}{2} \epsilon^2 \leq \frac{5}{2} \epsilon K L. 
%\]
\end{proof}

We can now prove the theorem which will be used in the induction step to prove that Algorithm~\ref{sepnmf} works under small perturbations of the input separable matrix. 

\begin{theorem} \label{Th2} 
Let 
\begin{itemize} 

\item $f$ satisfy Assumption~\ref{fass1}, with  strong convexity parameter $\mu$, and its gradient have Lipschitz constant $L$.  

\item $M' = M + N$ with $M = YH = [W \, Q] H$, where $W \in \mathbb{R}^{m \times k}$, $K(N) \leq \epsilon$, $\nu(W) > 2 \sqrt{\frac{L}{\mu}} K(Q)$, and $H = [I, H'] \in \mathbb{R}^{r \times n}_+$ where the sum of the entries of the columns of $H'$ is at most one. We will denote $\omega(W)$ and $K(W)$, $\omega$ and $K$ respectively. 
 
\item $\epsilon$ be sufficiently small so that $\epsilon \leq \frac{\mu {\omega}^2}{20 K L}$. 

\end{itemize}
Then the index $i$ corresponding to a column $m'_i$ of $M'$ that maximizes the function $f$  satisfies 
\begin{equation} \label{eqmdel2} 
m_i = [W, Q]h_i, %= \sum_{j=1}^k h_i(j) w_j, 
\quad \text{ where }  h_i(p) \geq 1 - \delta \text{ for some }  1 \leq p \leq k, 
\end{equation} 
and  $\delta =  \frac{10  \epsilon K L}{\mu \omega^2}$, 
which implies  
\begin{equation} \label{errbound}
||m'_i - w_p||_2 
\leq \epsilon + 2 K \delta
 = \epsilon \left( 1 + 20 \frac{K^2}{\omega^2} \frac{L}{\mu} \right).   
\end{equation} 
\end{theorem}
\begin{proof}
First note that $\epsilon \leq \frac{\mu \omega^2}{20 K L}$ implies $\delta =  \frac{10  \epsilon K L}{\mu \omega^2} \leq \frac{1}{2}$. 
%and, since $\frac{L}{\mu} \geq 1$ and $\frac{K}{\omega} \geq 1$,  $\epsilon \leq \frac{\omega}{10} \leq \frac{\nu}{10}$ hence $\nu - \epsilon \geq \frac{9}{10} \nu$. 
Let us then prove Equation~\eqref{eqmdel2} by contradiction. Assume the extracted index, say $i$, for which $m'_i = m_i + n_i = Yh_i + n_i$ satisfies $h_i(l) < 1 - \delta$ for $1 \leq l \leq k$.  
%that is, there exists $\delta' > \delta$ such that $h(l) \leq 1 - \delta' \, \forall \, l$. 
%(Notice that, by assumption on $M'$, $\sum_{l=1}^r h_i(l) \leq 1$.) 
We have 
%and $0 \leq \delta \leq \frac{1}{2}$. This implies 
\begin{align} 
f(m'_i) & \hspace{0.55cm} = \hspace{0.55cm} 
f(m_i + n_i) \nonumber \\
& \underset{(Lemma~\ref{fbound})}{\leq}
 f(Yh_i) + \frac{3}{2} \epsilon K L   \nonumber \\ 
&  \hspace{0.55cm} < \hspace{0.55cm}  \max_{x \in \Delta^r, x(l) < 1-\delta \, 1 \leq l \leq k} f(Yx)  
 + \frac{3}{2} \epsilon K L  \nonumber \\ 
& \underset{(Lemma~\ref{lemscf})}{\leq}
\max_j f(w_j) - \frac{1}{2} \mu \delta (1-\delta) \omega^2 + \frac{3}{2} \epsilon K L \nonumber \\
& \underset{(Lemma~\ref{fbound})}{\leq}
 \max_j f(w'_j) -  \frac{1}{2} \mu \delta (1-\delta) \omega^2 + \frac{5}{2} \epsilon K L, \label{eqmp} 
%*** u_f(K+\epsilon,\epsilon) 
\end{align} 
where $w'_j$ is the perturbed column of $M$ corresponding to $w_j$ (that is, the $j$th column of $M'$). 
The first inequality follows from Lemma~\ref{fbound}. 
In fact, we have $\epsilon \leq K$ since $\mu \leq L$ and $\omega \leq K$,
$||m_i||_2 = ||Wh_i||_2 \leq \max_i ||w_i||_2 = K$ (by convexity of $||.||_2$), and $||n_i||_2 \leq \epsilon$ $\forall i$ so that $f(m'_i) \leq f(m_i) + \frac{3}{2} \epsilon K L$. 
The second inequality is strict since the maximum is attained at a vertex with $x(l) = 1-\delta$ for some $1 \leq l \leq k$ at optimality (see proof of Lemma~\ref{lemscf}). 
 The third inequality follows from  Lemma~\ref{lemscf} while the fourth follows from the fact that $||w_j||_2 \leq K$ so that $f(w_j) - \epsilon K L \leq f(w'_j)$ for all $j$ by Lemma~\ref{fbound}. 

 We notice that, since $\delta \leq \frac{1}{2}$, 
\begin{equation} \nonumber 
 \frac{1}{2} \mu \delta (1-\delta) \omega^2 
\geq   \frac{1}{4} \mu \omega^2 \delta = \frac{1}{4} \mu \omega^2 \frac{10  \epsilon K L}{\mu \omega^2}
=  \frac{5}{2} \epsilon K L. 
\end{equation}
Combining this inequality with Equation~\eqref{eqmp}, we obtain $f(m'_i) < \max_j f(w'_j)$, a contradiction since $m'_i$ should maximize $f$ among the columns of $M'$ and the $w'_j$'s are among the columns of $M'$.

To prove Equation~\eqref{errbound}, we use Equation~\eqref{eqmdel2} and observe that 
\begin{equation} \nonumber
m_i = (1-\delta') w_p + \sum_{k\neq p} \alpha_k y_k \text{ for some $p$ and } 1-\delta' \geq 1-\delta,  
\end{equation} 
so that $\sum_{k\neq p} \alpha_k \leq \delta' \leq \delta$. 
Therefore, 
\begin{align*}
\left\|m_i - w_p\right\|_2 
& = \left\| - \delta' w_p + \sum_{k\neq p} \alpha_k w_k \right\|_2 \\ 
& \leq 2 \delta' \max_j ||w_j||_2 \leq 2 \delta' K \leq 2 K \delta, 
\end{align*} 
 which gives
\[
||m'_i - w_p||_2 \leq ||(m'_i-m_i) + (m_i - w_p)||_2 \leq \epsilon + 2 K \delta, 
\]
for some $1 \leq p \leq k$.  
\end{proof}

It is interesting to relate the ratio $\frac{K(W)}{\omega(W)}$ to the condition number  of matrix $W$, given by the ratio of its largest and smallest singular values $\kappa(W) = \frac{\sigma_1(W)}{\sigma_r(W)}$.

\begin{lemma} \label{gv1}
Let $W = [w_1 \, w_2 \, \dots \, w_r] \in \mathbb{R}^{m \times r}$. 
Then 
\[
\omega(W) \geq \sigma_r(W) \, .
\]
\end{lemma}
\begin{proof} 
We have to show that $||w_i||_2 \geq \sigma_r(W) \text{ for all } i$, and  
$||w_i - w_j||_2 \geq \sqrt{2} \sigma_r(W)$  for all $i \neq j$. 
Let $(U,\Sigma,V) \in \mathbb{R}^{m \times r} \times \mathbb{R}^{r \times r} \times \mathbb{R}^{r \times r}$ be a compact singular value decomposition of $W = U \Sigma V^T$, where $U$ and $V$ are orthonormal and $\Sigma$ is diagonal with the singular values of $W$ on the diagonal. Then 
\[
||w_i||_2 = || U \Sigma v_i||_2 = ||\Sigma v_i||_2 \geq \sigma_r(W) ||v_i||_2 =  \sigma_r(W), 
\]
while 
\begin{align*}
||w_i - w_j||_2 & = || U \Sigma (v_i-v_j)||_2   = ||\Sigma (v_i-v_j)||_2 \\ 
& \geq  \sigma_r(W) ||v_i-v_j||_2 =   \sqrt{2}  \sigma_r(W). 
\end{align*}
\end{proof}

The ratio $\frac{K(W)}{\omega(W)}$ is then closely related to the conditioning of matrix $W \in \mathbb{R}^{m \times r}$. In fact, we have seen that $\omega(W) \geq \sigma_r(W)$ while, by definition, $\sigma_1(W) \geq K(W) \geq \nu(W) \geq \omega(W)$ so that  
\[
1 \leq \frac{K(W)}{\omega(W)} \leq \frac{\sigma_1(W)}{\sigma_r(W)} = \kappa(W). 
\]
In particular, this inequality implies that if $\kappa(W) = 1$ then $\frac{K(W)}{\omega(W)} = 1$.

\subsubsection{Error Bound for Algorithm~\ref{sepnmf}}

We have shown that, if the input matrix $M'$ has the form 
\[
M' = [W, Q] [I_r, H'] + N, 
\]
where $Q$ and $N$ are sufficiently small and the sum of the entries of each column of $H' \geq 0$ is smaller than one, %is equal to a separable matrix (see Assumption~\ref{asssep}) plus some noise, 
then Algorithm~\ref{sepnmf} extracts one column of  $M'$ which is close to a column of $W$; cf.\@ Theorem~\ref{Th2}.   
We now show that, at each step of Algorithm~\ref{sepnmf}, the residual matrix satisfies these assumptions so that we can prove the result by induction.

We first give some useful lemmas; see \cite{GV96} and the references therein.  

%Improve the lemmas... using the fact that $u$ corresponds to a column of $W$ with large norm!? 

\begin{lemma}[Cauchy Interlacing Theorem] \label{gv2}  
Let $W \in \mathbb{R}^{m \times r}$ and $P = \prod_{i=1}^k (I-u_iu_i^T)$ where $u_i \in \mathbb{R}^{m}$ with $||u_i||_2 = 1$ for all $i$, and $k \leq r-1$. Then 
\[
\sigma_{1}(W) \geq \sigma_{1}(P W) \geq \sigma_{r-k}(P W) \geq \sigma_r(W). 
\]
\end{lemma}

\begin{lemma}[Singular Value Perturbation, Weyl] \label{weyl}
Let $M' = M + N \in \mathbb{R}^{r \times n}$ with $r \leq n$. Then, for all $1 \leq i \leq r$, 
\[
\left| \sigma_i(M) - \sigma_i(M') \right| \; \leq \; \sigma_1(N) = ||N||_2 \, . 
\]
\end{lemma}

\begin{lemma} \label{lemsv} 
Let $W = [W_1,\, W_2] \in \mathbb{R}^{m \times r}$ where $W_1 \in \mathbb{R}^{m \times r_1}$ and $W_2 \in \mathbb{R}^{m \times r_2}$.  
Let also $P = \prod_{i=1}^{r_2}\left(I-\frac{u_iu_i^T}{||u_i||_2^2}\right)$ be such that $\max_i ||PW_2(:,i)||_2 \leq \bar{\epsilon}$ for some $\bar{\epsilon} \geq 0$. 
Then, 
\[
\sigma_{r_1}(PW_1) \geq \sigma_r(W) - \sqrt{r_2} \, \bar{\epsilon}. 
\]
\end{lemma}
\begin{proof} We have 
\begin{align*}
\sigma_{r_1}(PW_1) 
&  \hspace{0.6cm} = \hspace{0.6cm} \sigma_{r_1}([PW_1, 0_{m \times r_2}]) \\
&  \underset{(Lemma~\ref{weyl})}{\geq} \sigma_{r_1}([PW_1, PW_2]) - ||PW_2||_2 \\ 
&  \hspace{0.6cm} \geq \hspace{0.6cm}   \sigma_{r_1}(PW) - \sqrt{r_2} \max_i ||PW_2(:,i)||_2 \\ 
&  \underset{(Lemma~\ref{gv2})}{\geq} \sigma_{r}(W) - \sqrt{r_2} \bar{\epsilon} \, .
\end{align*}
The second inequality follows from the fact that $||A||_2 \leq \sqrt{n} \max_i ||A(:,i)||_2$ for any matrix $A \in \mathbb{R}^{m \times n}$. %The inequality $||PW_2(:,i)||_2 \leq \bar{\epsilon}$ for all $i$ follows from $P$ being the orthogonal projection of the space spanned by the columns of $U$ and $||u_i-W_2(:,i)||_2 \leq \bar{\epsilon}$ for all $i$. 
%while the third follows from Lemma~\ref{gv2} and the fact that $||PW_2(:,i)||_2 \leq \bar{\epsilon}$ since $||u_i-W_2(:,i)||_2 \leq \bar{\epsilon}$ for all $i$. 
\end{proof}

We can now prove the main theorem of the paper which shows that, given a noisy separable matrix $M' = M+N = WH+N$ where $M$ satisfies Assumption~\ref{asssep}, Algorithm~\ref{sepnmf} is able to identify approximately the columns of $W$.

\begin{theorem} \label{threc} 
Let $M' = M + N = W H + N \in \mathbb{R}^{m \times n}$ where $M$ satisfies  Assumption~\ref{asssep} with $W \in \mathbb{R}^{m \times r}$, $r \geq 2$, and $H = [I_r \; H']$, and let $f$ satisfy Assumption~\ref{fass1} with  strong convexity parameter $\mu$ and its gradient has Lipschitz constant $L$. Let also $||n_i||_2 \leq \epsilon$ for all $i$ with  
\begin{equation} \label{sepnmfbound} 
\epsilon < \sigma_r(W) \, 
\min\left( \frac{1}{2 \sqrt{r-1}},\frac{1}{4}\sqrt{\frac{\mu}{L}}\right) 
\left( 1 + 80   \frac{K(W)^2}{\sigma_r^2(W)}  \frac{L}{\mu} \right)^{-1}, 
%\epsilon \leq \frac{1}{2(r-1)} \frac{\sigma_r^2(W)}{K(W)}
%\qquad \text{ and } \qquad 
%\epsilon < \sigma_r(W) \, \frac{1}{4}\sqrt{\frac{\mu}{L}} \left( 1 + 80   \frac{K(W)^2}{\sigma_r^2(W)}  \frac{L}{\mu} \right)^{-1} 
%, \frac{\mu}{20 \kappa(W) L}
%\right\}  . 
%\frac{\sigma_r(W)\mu}{{10 \kappa(W)}L} \; \;   
 %{\left( 1 + 160   \kappa^2(W) \frac{L}{\mu} \right)^{-1}}.  
\end{equation}
and $J$ be the index set of cardinality $r$ extracted by Algorithm~\ref{sepnmf}. Then there exists a permutation $P$ of $\{1,2,\dots,r\}$ such that %for  $j = 1, 2, \dots, r$, 
\begin{align*}
\max_{1 \leq j \leq r} ||m'_{J(j)} - w_{P(j)} ||_2 
& \leq \bar{\epsilon} 
= \epsilon \left(1+ 80 \frac{K(W)^2}{\sigma_r^2(W)} \frac{L}{\mu}\right) .  %\\ 
%& <  \sigma_r(W) \, 
%\min\left( \frac{1}{2(r-1)},\frac{1}{4}\sqrt{\frac{\mu}{L}}\right) . 
\end{align*}
%\epsilon^{(k)} \left( \frac{11}{4}+ 7 \kappa^2(W) \right)
\end{theorem}
\begin{proof} Let us prove the result by induction. First, let us define the residual matrix $R^{(k)}$ obtained after $k$ steps of Algorithm~\ref{sepnmf} as follows: 
\[
R^{(0)} = M', \qquad R^{(k+1)} = P^{(k)} R^{(k)} \; \text{ for } 1 \leq k \leq r-1 , 
\]
 with $P^{(k)} = \left(I-\frac{uu^T}{||u||_2^2}\right)$ is the orthogonal projection performed at step 5 of Algorithm~\ref{sepnmf} where $u$ is the extracted column of $R^{(k)}$, that is, $u = r^{(k)}_i$ for some $1 \leq i \leq n$.  \vspace{0.1cm}
 
Then, let us assume that the residual $R^{(k)} \in \mathbb{R}^{m \times n}$ has the following form
\[
R^{(k)} = [W^{(k)}, \, Q^{(k)}] H + N^{(k)}, 
\] 
where $W^{(k)} \in \mathbb{R}^{m \times (r-k)}$, $Q^{(k)} \in \mathbb{R}^{m \times k}$,  $K(Q^{(k)}) \leq \bar{\epsilon}$, $\nu(W^{(k)}) > 2 \sqrt{\frac{L}{\mu}} \bar{\epsilon}$, and $K(N^{(k)}) \leq \epsilon \leq \frac{\omega(W^{(k)})^2 \mu}{20 K(W^{(k)}) L }$. 
Let us show that $R^{(k+1)}$ satisfies the same conditions as $R^{(k)}$. 
%that is, $K(E^{(k+1)}) \leq \bar{\epsilon}$, $\nu(W^{(k+1)}) \geq 2 \sqrt{\frac{L}{\mu}} \bar{\epsilon}$, and $K(N^{(k+1)}) \leq \epsilon \leq \frac{\omega(W^{(k+1)})^2 \mu}{80 K(W^{(k+1)}) L }$. 
By assumption, $R^{(k)}$ satisfies the conditions of Theorem~\ref{Th2}: the $(k+1)$th index extracted by Algorithm~\ref{sepnmf}, say $i$, satisfies 
\begin{align*}
||r^{(k)}_i - w^{(k)}_p||_2 
&
\leq \epsilon \left(1 + 20 \frac{K(W^{(k)})^2}{\omega(W^{(k)})^2} \frac{L}{\mu}\right),  
%\leq \bar{\epsilon} = \epsilon \left(1 + 80 \kappa(W)^2 \frac{L}{\mu}\right)
\end{align*} 
for some $1 \leq p \leq r-k$. 
Let us assume without loss of generality that $p = r-k$. 
The next residual $R^{(k+1)}$ has the form 
\begin{align} \label{resi}
R^{(k+1)} 
%& = P^{(k)} R^{(k)} 
& = [\underbrace{P^{(k)}W^{(k)}(:,1\text{:}r\text{$-$}k\text{$-$}1)  }_{W^{(k+1)}}, \; 
\underbrace{P^{(k)}w^{(k)}_{p}, \, P^{(k)}Q^{(k)}}_{Q^{(k+1)}}] H \nonumber \\ 
& \hspace{0.7cm} + \underbrace{P^{(k)}N^{(k)}}_{N^{(k+1)}} \; ,  
%= [W^{(k+1)}, \, E^{(k+1)}] H + N^{(k+1)}.  
\end{align} 
where 
\begin{itemize}

\item $K(N^{(k+1)}) \leq K(N^{(k)}) \leq \epsilon$ and $K(P^{(k)}Q^{(k)}) \leq K(Q^{(k)}) \leq \bar{\epsilon}$ because an orthogonal projection can only reduce the $\ell_2$-norm of a vector.

\item $||P^{(k)}w^{(k)}_{p}||_2 \leq \bar{\epsilon} = \epsilon \left(1 + 80 %\frac{K(W^{(k)})^2}{\omega(W^{(k)})^2}
\frac{K(W)^2}{\sigma_r^2(W)} \frac{L}{\mu}\right)$ since 
\begin{align*}
||P^{(k)}w^{(k)}_{p}||_2 & 
\hspace{0.4cm} \leq \hspace{0.4cm} 
||r^{(k)}_i - w^{(k)}_p||_2 \\ 
& \underset{(Thm~\ref{Th2})}{\leq}  
\epsilon \left(1 + 20 \frac{K(W^{(k)})^2}{\omega(W^{(k)})^2} \frac{L}{\mu}\right)  
%\leq \bar{\epsilon} = \epsilon \left(1 + 80 \kappa(W)^2 \frac{L}{\mu}\right)
\end{align*}
where the first inequality follows from $P^{(k)}w^{(k)}_{p}$ being the projection of $w^{(k)}_p$ onto the orthogonal complement of $r^{(k)}_i$. Moreover, 
\begin{equation} \label{Kws}
\frac{K(W^{(k)})^2}{\omega(W^{(k)})^2} \leq 4 \frac{K(W)^2}{\sigma_r^2(W)}. 
\end{equation} 
In fact, $K(W^{(k)}) \leq K(W)$ because of the orthogonal projections, while  $\omega(W^{(k)}) \geq \frac{1}{2} \sigma_r(W)$ follows from 
\begin{align*}
%\nu(W^{(k)}) \geq 
\omega(W^{(k)}) 
& \underset{(Lemma~\ref{gv1})}{\geq} \sigma_{r-k}(W^{(k)}) \\ 
%& \geq  \sigma_{r-k+1}(W^{(k-1)}) - \bar{\epsilon} \\ 
&  \underset{(Lemma~\ref{lemsv})}{\geq} \sigma_r(W) - \sqrt{k} \bar{\epsilon} \\ 
&  \hspace{0.6cm} \geq \hspace{0.6cm} \sigma_r(W) - \sqrt{r-1} \bar{\epsilon} \\
&  \hspace{0.6cm} \geq \hspace{0.6cm} \frac{1}{2} \sigma_r(W). 
%\geq 2 \sqrt{\frac{L}{\mu}} \bar{\epsilon} , 
\end{align*}
Lemma~\ref{lemsv} applies since $K(Q^{(k)}) \leq \bar{\epsilon}$. 
The last inequality %follows from Lemma~\ref{gv1}, the second and third from Lemma~\ref{lemsv}, while the last 
follows from $\bar{\epsilon} \leq \frac{\sigma_r(W)}{2 \sqrt{r-1}}$ since 
\[ 
\epsilon \leq \frac{\sigma_r(W)}{2\sqrt{r-1}} \left( 1 + 80  \frac{K(W)^2}{\sigma_r^2(W)}  \frac{L}{\mu} \right)^{-1}.  
\]
\end{itemize} 

 For $R^{(k+1)}$ to satisfy the same conditions as $R^{(k)}$, it remains to show that $\nu(W^{(k+1)}) > 2 \sqrt{\frac{L}{\mu}} \bar{\epsilon}$ and  $\epsilon \leq \frac{\omega(W^{(k+1)})^2 \mu}{20 K(W^{(k+1)}) L }$. Let us show that these hold for all $k =  0, 1, \dots, r-1$ :  
\begin{itemize}
\item Since $\nu(W^{(k)}) \geq \omega(W^{(k)}) \geq \frac{1}{2}\sigma_r(W)$ (see above), $\nu(W^{(k)}) > 2 \sqrt{\frac{L}{\mu}} \bar{\epsilon}$ %for $0 \leq k \leq r-1$ 
is implied by $\frac{1}{2} \sigma_r(W) > 2  \sqrt{\frac{L}{\mu}} \bar{\epsilon}$, that is,   
\[
\epsilon < \frac{1}{4} \sqrt{\frac{\mu}{L}} \sigma_r(W) \left( 1 + 80   \frac{K(W)^2}{\sigma_r^2(W)} \frac{L}{\mu} \right)^{-1}. 
\]

\item Using Equation~\eqref{Kws}, we have 
\begin{align*}
\epsilon & \leq \frac{\sigma_r(W)}{2\sqrt{r-1}} \left( 1 + 80  \frac{K(W)^2}{\sigma_r^2(W)}  \frac{L}{\mu} \right)^{-1} \\
& \leq \sigma_r(W) \frac{\sigma_r(W) \mu}{80 K(W) L } \leq \frac{\omega(W^{(k)})^2 \mu}{20 K(W^{(k)}) L } . 
\end{align*}
\end{itemize}

By assumption on the matrix $M'$, these conditions are satisfied at the first step of the algorithm (we actually have that $Q^{(0)}$ is an empty matrix), so that, by induction, all the residual matrices satisfy these conditions. 
Finally, Theorem~\ref{Th2} implies that the index $i$ extracted by Algorithm~\ref{sepnmf} at the $(k+1)$th step  satisfies 
\[
r^{(k)}_i = [W^{(k)}, Q^{(k)}] h_i + n_i %\sum_{j=1}^{r-k} h_i(j) w^{(k)}_j
, \quad \text{ where }  h_i(p) \geq 1 - \delta^{(k)},
\] 
$\delta^{(k)} = \frac{10 \epsilon K(W^{(k)})L}{\omega(W^{(k)})^2 \mu}$, and  $p = r-k$ without loss of generality. Since the matrix $H$ is unchanged between each step, this implies 
%\[
$m_i = W h_i$  where   $h_i(p) \geq 1 - \delta^{(k)}$,  
%\] 
hence  
\begin{align*}
||m'_i - w_p||_2 & \hspace{0.4cm} = \hspace{0.4cm} ||m'_i - m_i + m_i - w_p||_2 \\ 
& \hspace{0.4cm} \leq \hspace{0.4cm} \epsilon + ||m_i - w_p||_2 \\
%& \hspace{0.4cm} \leq \hspace{0.4cm} \epsilon + ||w_p h_i(p)  - w_p||_2 \\
&\hspace{0.4cm} \leq \hspace{0.4cm} \epsilon + 2 \delta^{(k)} K(W) \\
& \hspace{0.4cm} = \hspace{0.4cm} \epsilon + 20 \frac{\epsilon K(W^{(k)})L}{\omega(W^{(k)})^2 \mu} K(W) \\
& \underset{(Eq.~\ref{Kws})}{\leq} \epsilon \left(1 + 80 \frac{K(W)^2}{\sigma_r^2(W)} \frac{L}{\mu}\right). 
\end{align*} 
The second inequality is obtained using $m_i = Wh_i = h_i(p) w_p + \sum_{k \neq p} h_i(k) w_k$ and $\sum_{k \neq p} h_i(k) \leq 1-h_i(p)$ so that 
\begin{align*}
||w_p - m_i||_2 & = ||(1- h_i(p)) w_p - \sum_{k \neq p} h_i(k) w_k ||_2 \\
&  \leq 2 (1- h_i(p))\max_k ||w_k||_2  \leq 2 \delta^{(k)} K(W). 
\end{align*} 
\end{proof}

\subsection{Bounds for Separable NMF and Comparison with the Algorithms of Arora~et~al.~\cite{AGKM11} and Bittorf~et~al.~\cite{BRRT12}} \label{compa}

Given a noisy separable matrix \mbox{$M' = M + N = W H + N$}, we have shown that Algorithm~\ref{sepnmf} is able to approximately recover the columns of the matrix $W$. In order to compare the bound of Theorem~\ref{threc} with the ones obtained in \cite{AGKM11} and\footnote{The error bounds in \cite{BRRT12} were only valid for separable matrices without duplicates of the columns of $W$. They were later improved and generalized to any separable matrix in \cite{G13}.} \cite{BRRT12}, let us briefly recall  the results obtained in both papers: Given a separable matrix $M = WH$, %the authors define 
the parameter $\alpha$ is defined as the minimum among the $\ell_1$ distances between each column of $W$ and its projection onto the convex hull of the other columns of $W$. 
Without loss of generality, it is assumed that the $\ell_1$ norm of the columns of $W$ is equal to one (by normalizing the columns of $M$), hence $\alpha \leq 2$. 
Then, given that the noise $N$ added to the separable matrix $M$ satisfies  
\[
\epsilon_{1} = \max_i ||n_i||_1 \leq \mathcal{O}({\alpha^2}) , 
\]
Arora et al.\@ \cite{AGKM11} identify a matrix $U$ such that 
\[
\max_{1 \leq k \leq r} \min_{1 \leq j \leq r} ||u_j - w_k||_1 \leq \mathcal{O}\left( \frac{\epsilon_{1}}{\alpha} \right). %  10 \frac{\epsilon_{1}}{\alpha} + 7 \epsilon_1, 
% \mathcal{O}\left(\frac{\epsilon}{\alpha} + \epsilon\right), 
\]
 The algorithm of Bittorf et al.\@ \cite{BRRT12} identifies a matrix $U$ satisfying 
 \[
 \max_{1 \leq k \leq r} \min_{1 \leq j \leq r} ||u_j - w_k||_1  \leq \mathcal{O}\left( \frac{r \, \epsilon_{1}}{\alpha} \right), 
 \] 
 given that $\epsilon_{1} \leq \mathcal{O}\left(\frac{\alpha \gamma_1}{r}\right)$ where $\gamma_1 = \min_{i \neq j} ||w_i - w_j||_1 \geq \alpha$;  see \cite{G13}. 
Let us compare these bounds to ours. Since the $\ell_1$ norm of the columns of $W$ is equal to one, we have $\sigma_r(W) \leq K(W) \leq 1$. Moreover, denoting $p_i$ the orthogonal projection of $w_i$ onto the linear subspace generated by the columns of $W$ but the $i$th, we have 
\[
\sigma_r(W) \leq \min_i ||w_i-p_i||_2 \leq \min_i ||w_i-p_i||_1 \leq \alpha. 
\]
By Theorem~\ref{threc}, Algorithm~\ref{sepnmf} therefore requires  
\[
\epsilon_2 = \max_i ||n_i||_2 
\leq \mathcal{O}\left(\frac{\sigma_r^3(W)}{\sqrt{r}} \right) 
\leq \mathcal{O}\left(\frac{\alpha^3}{\sqrt{r}} \right) ,  
\]
to obtain a matrix $U$ such that 
\[
\max_{1 \leq k \leq r} \min_{1 \leq j \leq r} ||u_j - w_k||_2 
\leq \mathcal{O}\left(\frac{\epsilon_2}{ \sigma_r^{2}(W)}\right) . 
%\quad \text{ where } \mathcal{O}\left(\frac{\epsilon}{ \sigma_r^{2}(W)}\right) \geq \mathcal{O}\left(\frac{\epsilon}{ \alpha^2}\right) . 
\] 
This shows that the above bounds are tighter, as they only require the noise to be bounded above by a constant proportional to~$\alpha^2$ to  guarantee an NMF with error proportional to~$\frac{\epsilon_1}{\alpha}$. %(up to some factor in $r$, usually negligible in practice). 
In particular, if $W$ is not full rank, Algorithm~\ref{sepnmf} will fail to extract more than $\rank(W)$ columns of $W$, while the value of $\alpha$ can be much larger than zero implying that the algorithms from \cite{AGKM11, BRRT12} will still be robust to a relatively large noise. 
%(Notice however that the condition on the noise for \cite{AGKM11, BRRT12} is based on the $\ell_1$ norm, hence is stronger than ours --possibly up to a factor $\sqrt{m}$, but their error bound also in \ell_1 so... .)  
 %does not require the matrix $W$ to be full rank. 

To conclude, the techniques in \cite{AGKM11, BRRT12} based on linear programming lead to better error bounds.
However, there are computationally much more expensive (at least quadratic in $n$, while Algorithm~\ref{sepnmf} is linear in $n$, cf.\@ Section~\ref{pwork}), and have the drawback that some parameters have to be estimated in advance: the noise level $\epsilon_1$, and 
\begin{itemize}
\item the parameter $\alpha$ for Arora et al.\@ \cite{AGKM11} (which is rather difficult to estimate as $W$ is unknown), 
\item the factorization rank $r$ for Bittorf et al.\@ \cite{BRRT12}\footnote{In their incremental gradient descent algorithm, the parameter $\epsilon$ does not need to be estimated. However, other parameters need to be tuned, namely, primal and dual step sizes.},   
%Note also that their algorithm is not deterministic as it requires to choose a vector of dimension $n$ with different entries.} 
hence the solution has to be recomputed from scratch when the value of $r$ is changed (which often happens in practice as the number of columns to be extracted is typically estimated with a trial and error approach). 

%which can be quite problematic for large- and huge-scale problems since we typically don't know the number of factors to be extracted). 
\end{itemize} 
Moreover, the algorithms from \cite{AGKM11, BRRT12} heavily rely on the separability assumption while Algorithm~\ref{sepnmf} still makes sense even if the separability assumption is not satisfied; see Section~\ref{l2norm}. 
%Moreover, we will see in Section~\ref{compaB} that Algorithm~\ref{sepnmf} performs in average better on several synthetic data sets. 
Table~\ref{comptable} summarizes these results. Note that we keep the analysis simple and only indicate the growth in terms of $n$.   
%it is dominant for both algorithms \cite{AGKM11,BRRT12}. Moreover, 
The reason is threefold: (1) in many applications (such as hyperspectral unmixing), $n$ is much larger than $m$ and $r$, (2) a more detailed comparison of the running times would be possible (that is, in terms of $m$, $n$, and $r$) but is not straightforward as it depends on
% parameters that are not directly comparable (in particular, $r$ versus $\alpha$) and on 
the algorithm used to solve the linear programs (and possibly on the parameters $\alpha$ and $\epsilon_1$), and (3) both algorithms \cite{AGKM11,BRRT12} are at least quadratic in $n$ (for example, the computational cost of each iteration of the first-order method proposed in \cite{BRRT12} is proportional to $mn^2$, so that the complexity is linear in $m$).  
%(and also on other parameters such as $\alpha$ and $\epsilon_1$). 
\begin{table}[ht!] 
\begin{center}
\caption{Comparison of robust algorithms for separable NMF.} 
%(columns of the input matrix $M$ are assumed to be normalized so that $\sigma_r(W) \leq \alpha$).} 
\begin{tabular}{|c||c|c|c|}
\hline
					& Arora et al. \cite{AGKM11}  &   Bittorf et al. \cite{BRRT12} & Algorithm~\ref{sepnmf}   \\ 
\hline  \hline  
Flops    &	 $\Omega(n^2)$   &   $\Omega(n^2)$ & $\mathcal{O}(n)$  \\ [5pt] 
Memory    &	 $\mathcal{O}(n)$   &   $\mathcal{O}(n^2)$ & $\mathcal{O}(n)$  \\ [5pt] 
Noise    &	 $\epsilon_1 \leq \mathcal{O}(\alpha^2)$   &  $\epsilon_1 \leq \mathcal{O}\left(\frac{\alpha \gamma_1}{r}\right)$  
&  $\epsilon_2 \leq \mathcal{O}\left( \frac{\sigma_r^3(W)}{\sqrt{r}}\right)$  \\ [5pt] 
Error  &	 $||.||_1 \leq \mathcal{O}\left(\frac{\epsilon_1}{\alpha}\right)$   
&  \hspace{-0.2cm} $||.||_1 \leq \mathcal{O}\left(r \frac{\epsilon_1}{\alpha}\right)$ \hspace{-0.2cm} & 
\hspace{-0.2cm} $||.||_2 \leq \mathcal{O} \left(  \frac{\epsilon_2}{\sigma_r^2(W)}\right)$ \hspace{-0.3cm} \\ [5pt] % \frac{L}{\mu}? 
\hspace{-0.2cm} Parameters \hspace{-0.2cm}    &	 $\epsilon_1$, $\alpha$   &  $\epsilon_1$, $r$ & $r$  \\ 
\hline  
\end{tabular}
\label{comptable}
\end{center}
\end{table}

\section{Outlier Detection} \label{outliers}

It is important to point out that Algorithm~\ref{sepnmf} is very sensitive to outliers, as are most algorithms aiming to detect the vertices of the convex hull of a set of points, e.g., the algorithms from \cite{AGKM11,BRRT12} discussed in the previous section.  
%see, e.g., \cite{OS05}. 
Therefore, one should ideally discard the outliers beforehand, or design variants of these methods robust to outliers. In this section, we briefly describe a simple way for dealing with (a few) outliers. This idea is  inspired from the approach described in \cite{ESV1294}.  

Let us assume that the input matrix $M$ is a separable matrix $WH = W[I, H']$ satisfying Assumption~\ref{asssep} to which was added $t$ outliers gathered in the matrix $T \in \mathbb{R}^{m \times t}$ : 
\begin{align}
M & = \left[ W, \, T, \, WH' \right] \nonumber \\ 
& = [ W, \, T ] 
\left[ \begin{array}{ccc} 
I_r & 0_{r \times t} & H' \\
0_{t \times r} & I_t   &  0_{t \times r}   \end{array} \right] \label{modelo}  \\  
 & = [W, T][I_{r+t}, F'] = [W, T] F ,  \nonumber 
\end{align}
where $h'_i \in \Delta^r$ for all $i$, hence $f'_i \in \Delta^r$ for all $i$. Assuming $[W, T]$ has rank $r+t$ (hence $t \leq m-r$), the matrix $M$ above also satisfies Assumption~\ref{asssep}. 
%then the matrix $M$ with outliers will also satisfy the condition of Theorem~\ref{threc}. Hence
 In the noiseless case, Algorithm~\ref{sepnmf} will then extract a set of indices corresponding to columns of $W$ and $T$ (Theorem~\ref{th1}). 
 Therefore, assuming that the matrix $H'$ has at least one non-zero element in each row, one way to identifying the outliers would be to 
\begin{itemize}

\item[(1)] Extract $r+t$ columns from the matrix $M$ using Algorithm~\ref{sepnmf}, 

\item[(2)] Compute the corresponding optimal abundance matrix $F$, and 

\item[(3)] Select the $r$ columns corresponding to rows of $F$ with the largest sum,  %(these rows should be the ones corresponding to columns of  $W$), 

\end{itemize}
see Algorithm~\ref{sepnmfo}. (Note that Algorithm~\ref{sepnmfo} requires to solve a convex quadratic program, hence it is computationally much more expensive than Algorithm~\ref{sepnmf}.) 
\algsetup{indent=2em}
\begin{algorithm}[ht!]
\caption{Recursive algorithm for separable NMF with outliers \label{sepnmfo}}
\begin{algorithmic}[1]
%Separable matrix $M = WH \in \mathbb{R}^{m \times n}_+$ (see Assumption~\ref{asssep}),  the number $r$ of columns to be extracted, and a strongly convex function $f$ satisfying Assumption~\ref{fass1}. 
\REQUIRE A separable matrix $M \in \mathbb{R}^{m \times n}_+$ containing at most $t \leq m-r$ outliers, the number $r$ of columns to be extracted, and a strongly convex function $f$ satisfying Assumption~\ref{fass1}. 
%\ENSURE Set of indices $J$ such that $M(:,J) = W$ up to permutation and scaling.   
\ENSURE Set of indices $J'$ such that $M(:,J') = W$ up to permutation and scaling; cf.\@ Theorem~\ref{robout}. 
\medskip 
%\STATE Remove the zero columns from $M$. 
%\STATE Set the $\ell_1$-norm of the columns of $M$ to one. 
\STATE Compute $J$, the index set of cardinality $(r+t)$ extracted by Algorithm~\ref{sepnmf}. \vspace{0.1cm}
\STATE $G = \argmin_{x_i \in \Delta^{r+t} \, \forall i} %i = 1,2,\dots, n} 
||M-M(:,J)X||_F^2$. \vspace{0.1cm}
%\STATE $\text{score}_i = ||h_i||_1$. 
\STATE Compute $J' \subseteq J$, the set of $r$ indices having the largest score $||G(j,:)||_1$ among the indices in $J$. 
% \mathbb{R}^{(r+t) \times m} 
\end{algorithmic} 
\end{algorithm}  
It is easy to check that Algorithm~\ref{sepnmfo} will recover the $r$ columns of $W$ because the optimal solution $G$ computed at the second step is unique and equal to $F$ (since $[W,T]$ is full rank), hence $||G(j,$:$)||_1 > 1$ for the indices corresponding to the columns of $W$ while $||G(j,$:$)||_1 = 1$ for the outliers; see Equation~\eqref{modelo}. 

In the noisy case, a stronger condition is necessary: the sum of the entries of each row of $H'$ must be larger than some bound depending on the noise level.  
In terms of hyperspectral imaging, it means that for an endmember to be distinguishable from an outlier, its abundance in the image should be sufficiently large, which is perfectly reasonable.   

%and is left out (notably requires the sum of each row of $H'$ to be bounded below by some constant). 

\begin{theorem} \label{robout}
Let $M' = [W, T, WH'] + N \in \mathbb{R}^{m \times n}$ where $[W, WH']$ satisfies  Assumption~\ref{asssep} with $W \in \mathbb{R}^{m \times r}$, $T \in \mathbb{R}^{m \times t}$, $r\geq 2$, and let $f$ satisfy Assumption~\ref{fass1} with  strong convexity parameter $\mu$ and its gradient has Lipschitz constant $L$. Let also denote $B = [W, T]$, s = r+t, and $||n_i||_2 \leq \epsilon$ for all $i$ with  
\begin{equation*} %\label{sepnmfboundo} 
\epsilon < \sigma_s(B) \, 
\min\left( \frac{1}{2\sqrt{s-1}},\frac{1}{4}\sqrt{\frac{\mu}{L}}\right) 
\left( 1 + 80   \frac{K(B)^2}{\sigma_{s}^2(B)}  \frac{L}{\mu} \right)^{-1}, 
\end{equation*} 
and $J$ be the index set of cardinality $r$ extracted by Algorithm~\ref{sepnmfo}. 
If 
\begin{equation} \label{Hass} 
2 (2n-t-r)  \frac{ \bar{\epsilon} +  \epsilon}{\sigma_s(B)} < ||H'(i,:)||_1 %\leq n 
\quad \text{ for all $i$}, 
\end{equation}
then there exists a permutation $P$ of $\{1,2,\dots,r\}$ such that 
%for  $j = 1, 2, \dots, r$, 
\begin{align*}
\max_{1 \leq j \leq r}  ||m'_{J(j)} - w_{P(j)} ||_2 
& \leq \bar{\epsilon} 
= \epsilon \left(1+ 80 \frac{K(B)^2}{\sigma_{s}^2(B)} \frac{L}{\mu}\right) . 
% \\ & <  \sigma_r(B) \, \min\left( \frac{1}{2(s-1)},\frac{1}{4}\sqrt{\frac{\mu}{L}}\right) . 
\end{align*}
\end{theorem}
\begin{proof} 
By Theorem~\ref{threc}, the columns extracted at the first step of Algorithm~\ref{sepnmfo} correspond to the columns of $W$ and $T$ up to error $\bar{\epsilon}$. Let then $W+N_W$ and $T+N_T$ be the columns extracted by Algorithm~\ref{sepnmf} with $K(N_W), K(N_T) \leq \bar{\epsilon}$. 

At the second step of Algorithm~\ref{sepnmfo}, the matrix $G$ is equal to 
\[
G = \argmin_{x_i \in \Delta^{r+t} \, \forall i} ||M' - [W+N_W, T+N_T] X||_F^2, 
\]
up to the permutation of its rows. 
It remains to show that 
\begin{equation} \label{lub}
\min_{1 \leq i \leq r} ||G(i,:)||_1 \;  > \;  \max_{r+1 \leq i \leq r+t} ||G(i,:)||_1, 
\end{equation}
so that the last step of Algorithm~\ref{sepnmfo} will identify correctly the columns of $W$ among the ones extracted at the first step. We are going to show that 
\[
G \approx F = \left[ \begin{array}{ccc} 
I_r & 0_{r \times t} & H' \\
0_{t \times r} & I_t   &  0_{t \times r}   \end{array} \right]. 
\]
More precisely, we are going to prove the following lower (resp.\@ upper) bounds for the entries of the first $r$ (resp.\@ last $t$) rows of $G$:  
% ones in order to guarantee  \eqref{lub} to be satisfied: 
 \begin{itemize}
 \item[(a)] For $1 \leq i \leq r$, 
 $G_{ij} \geq \max\left(0, F_{ij} - 2 \frac{ \bar{\epsilon} +  \epsilon}{\sigma_s(B)}\right)$ for all $j$.  
 
 \item[(b)] For $r+1 \leq i \leq r+t$, $G_{ij} \leq \min\left(1, F_{ij} + 2 \frac{ \bar{\epsilon} +  \epsilon}{\sigma_s(B)}\right)$ for all $j$. 
 \end{itemize}
 Therefore, assuming (a) and (b) hold, we obtain 
 \begin{align*}
 \min_{1 \leq i \leq r} ||G(i,:)||_1  
 & \hspace{0.2cm} \underset{(a)}{\geq} \hspace{0.2cm} 
 \left(1-2  \frac{ \bar{\epsilon} +  \epsilon}{\sigma_s(B)}\right) + 
 ||H'(i,:)||_1 \\
 & \hspace{0.7cm} - 2 (n-t-r)  \frac{ \bar{\epsilon} +  \epsilon}{\sigma_s(B)} \\ 
 &  \underset{Eq.\eqref{Hass}}{>}   1 + 2 (n-1) \frac{ \bar{\epsilon} +  \epsilon}{\sigma_s(B)} \\ 
 & \hspace{0.2cm} \underset{(b)}{\geq} \hspace{0.2cm} \max_{r+1 \leq i \leq r+t} ||G(i,:)||_1, 
 \end{align*} 
 which proves the result. 
 It remains to prove (a) and (b). 
 %Let us analyze each column of $G$ independently. 
 First, we have that 
 \begin{equation} \label{lowbnd}
 % = \argmin_{y \in \Delta^{r+t}} 
 ||m'_j - [W+N_W, T+N_T] G_{:j}||_2 \leq \epsilon + \bar{\epsilon} \quad \text{ for all $j$}. 
 \end{equation}   
 In fact, for all $j$, 
 \begin{align*}
 ||m'_j - & [W+N_W, T+N_T] G_{:j}||_2  \\
  & \leq ||m_j + n_j - [W, T] F_{:j} - [N_W, N_T]F_{:j}||_2 \\
 & \leq ||m_j - [W, T] F_{:j}||_2 + \epsilon + \bar{\epsilon} = \epsilon + \bar{\epsilon}, 
 \end{align*}
 since $G(:,j)$ leads to the best approximation of $m'_j$ over $\Delta$ (see step 2 of Algorithm~\ref{sepnmfo}) and $f_j \in \Delta$.  
 
Then, let us prove the upper bound for the block of matrix $G$ at position $(2,1)$, that is, let us prove that  
\[
G_{ij} \leq 2 \frac{ \bar{\epsilon} +  \epsilon}{\sigma_s(B)}  
\quad \text{ for all } r+1  \leq i \leq r+t \text{ and } 1 \leq j \leq r. 
\]
(Note that $0 \leq G \leq 1$ by construction, hence some of the bounds are trivial, e.g., for the block (1,2).) The derivations necessary to obtain the bounds for the other (non-trivial) blocks are exactly the same and are then omitted here. 
Let $r+1  \leq i \leq t$ and $1 \leq j \leq r$ and denote $G_{ij} = \delta$, and let also  $I = \{1,2,\dots,r+t\} \backslash \{i\}$. We have  
\begin{align}
||m'_j - & [W+N_W, T+N_T] G_{:j}||_2 \nonumber \\ 
%& = || (w_j + n_j) + B(:,I)  G(I,j) + B(:,i) \delta + [N_W, N_T]  G_{:j}||_2 \\ 
& = || (w_j + n_j) + B G_{:j} + [N_W, N_T]  G_{:j}||_2 \nonumber \\ 
& \geq || w_j + B(:,I)  G(I,j) + B(:,i) \delta ||_2 -  \epsilon  - \bar{\epsilon}  \nonumber \\ 
& \geq \min_{y \in \mathbb{R}^{r+t-1}} \delta || B(:,I) y - B(:,i)  ||_2 -  \epsilon - \bar{\epsilon} \nonumber \\
& \geq \delta \sigma_s(B)  -  \epsilon  - \bar{\epsilon} . \label{dels} 
%& \geq \sqrt{2} \delta \sigma_s(B) -  \epsilon   - \bar{\epsilon} .
\end{align} 
The first inequality follows from $K(N) \leq \epsilon$, $K([N_W, N_T]) \leq \bar{\epsilon}$ and $G_{:j} \in \Delta^{r+t}$, while the second inequality follows from the fact that $w_j$ is a column of $B(:,I)$. The last inequality follows from the fact that the projection of any column of $B$ onto the subspace spanned by the other columns is at least $\sigma_s(B)$. 
 Finally, using Equations~\eqref{lowbnd} and~\eqref{dels}, we have $\bar{\epsilon} +  \epsilon \geq \delta \sigma_s(B)  - \bar{\epsilon} -  \epsilon$, hence $G_{ij} = \delta \leq 2 \frac{ \bar{\epsilon} +  \epsilon}{\sigma_s(B)}$. 
\end{proof}

\section{Choices for $f$ and Related Methods}  \label{choicef}

In this section, we discuss several choices for the function $f$ in Algorithm~\ref{sepnmf}, and relate them to existing methods.

\subsection{Best Choice with $\mu = L$: $f(x) = ||x||_2^2$} \label{l2norm}

According to our derivations (see Theorem~\ref{threc}), using functions $f$  whose strong convexity parameter $\mu$ is equal to the Lipschitz constant $L$ of its gradient is the best possible choice (since it minimizes the error bounds). 
The only function satisfying Assumption~\ref{fass1} along with this condition is, up to a scaling factor, 
%\[ 
 $f(x)  =  ||x||_2^2  = \sum_{i=1}^m x_i^2$.  
%\] 
In fact, Assumption~\ref{fass1} implies  $\frac{\mu}{2} ||x||_2^2 \leq f(x) \leq \frac{L}{2} ||x||_2^2$; see Equation~\eqref{normconv}. 
However, depending on the problem at hand, other choices could be more judicious (see Sections~\ref{otherchoice} and \ref{pnorms}). 
It is worth noting that Algorithm~\ref{sepnmf} with $f(x) = ||x||_2^2$ has been introduced and analyzed by several other authors: 

 \begin{itemize}
 
 \item \textbf{Choice of the Reflections for QR Factorizations.} 
 Golub and Businger \cite{BG65} construct QR factorizations of matrices by performing, at each step of the algorithm, the Householder reflection with respect to the column of $M$ whose projection onto the orthogonal complement of the previously extracted columns has maximum $\ell_2$-norm.

 \item \textbf{Successive Projection Algorithm.}  
  Ara\'ujo et al.\@ \cite{MC01} proposed the successive projection algorithm (SPA), which is equivalent to  Algorithm~\ref{sepnmf} with $f(x) = ||x||_2^2$. They used it for variable selection in spectroscopic multicomponent analysis, and showed it works better than other standard techniques. In particular, they mention `SPA seems to be more robust than genetic algorithms' but were not able to provide a rigorous justification for that fact (which our analysis does). %This algorithm has been used successfully by other authors for hyperspectral unmixing; see, e.g., \cite{ZR08}. 
 Ren and Chang \cite{RC03} rediscovered the same algorithm, which was referred to as the automatic target generation process (ATGP). It  was empirically observed in \cite{CW07} to perform better than other hyperspectral unmixing techniques (namely, PPI~\cite{B94} and VCA~\cite{ND05}). However, no rigorous explanation of their observations was provided. In Section~\ref{ne}, we describe these techniques and explain why they are not robust to noise, which theoretically justifies the better performances of Algorithm~\ref{sepnmf}. 
Chan et al.\@ \cite{CM11}  analyzed the same algorithm (with the difference that the data is preprocessed using a linear dimensionality reduction technique). 
 The algorithm is referred to as the successive volume maximization algorithm (SVMAX).  
 %and its optimality is proved in the noiseless case under Assumption~\ref{asssep}, that is, they proved Theorem~\ref{th1} in case $f(x) = ||x||_2^2$. 
 %Chan et al.\@ \cite{CM11} also introduced a `robust' variant of their  volume maximization formulation. However, they do not provide any recovery guarantee in the noisy case (nor an algorithm able to solve their problem up to global optimality). 
They also successfully use Algorithm~\ref{sepnmf} as an initialization for a more sophisticated approach which does not take into account the pure-pixel assumption.

\item \textbf{Greedy Heuristic for Volume Maximization.}   
\c{C}ivril and Magdon-Ismail~\cite{AM09, AM10} showed that Algorithm~\ref{sepnmf} with $f(x) = ||x||_2^2$ is a very good greedy heuristic for the following problem: given a matrix $M$ and an integer $r$, find a subset of $r$ columns of $M$ whose convex hull has maximum volume. More precisely, unless $\mathcal{P} = \mathcal{NP}$, they proved that the approximation ratio guaranteed by the greedy heuristic is within a logarithmic factor of the best possible achievable ratio by any  polynomial-time algorithm. 
However, the special case of separable matrices was not considered.  
This is another advantage of Algorithm~\ref{sepnmf}: even if the input data matrix $M$ is not approximately separable, it identifies $r$ columns of $M$ whose convex hull has large volume. For the robust algorithms from \cite{AGKM11, BRRT12} discussed in Section~\ref{compa}, it is not clear whether they will be able to produce a meaningful output in that case; see also Section~\ref{compaB} for some numerical experiments.

\end{itemize}

\subsection{Generalization: $f(x) = \sum_{i=1}^m h(x_i)$} \label{otherchoice}

Given a one-dimensional function $h : \mathbb{R} \to \mathbb{R}_+$ satisfying Assumption~\ref{fass1}, it is easy to see that the separable function $f(x) = \sum_{i=1}^m h(x_i)$ also satisfies Assumption~\ref{fass1} (notice that $h(y) = y^2$ gives $f(x) = ||x||_2^2$.). For example, we could take 
\[
h(y) = \frac{y^2}{\alpha + |y|} , \quad \alpha > 0, 
\; \text{ hence } f(x) = \sum_{i=1}^m \frac{x_i^2}{\alpha + |x_i|}. 
\] 
This choice limits the impact of large entries in $x$, hence would potentially be more robust to outliers. In particular, as $\alpha$ goes to zero, $f(x)$  converges to $||x||_1$ while, when $\alpha$ goes to infinity, it converges to $\frac{||x||_2^2}{\alpha}$ (in any bounded set). 

\begin{lemma}
In the ball $\{ y \in \mathbb{R} \ | \ |y| \leq K\}$, the function 
\[
h(y) = \frac{y^2}{\alpha + |y|} , \quad \alpha > 0, 
\]
is strongly convex with parameter $\mu = \frac{2 \alpha^2}{(\alpha+K)^3}$ and its gradient is Lipschitz continuous with constant $L = \frac{2}{\alpha}$. 
\end{lemma}
\begin{proof}
On can check that 
\[
\frac{2 \alpha^2}{(\alpha+K)^3} 
\leq 
h''(y) = \frac{2 \alpha^2}{(\alpha+|y|)^3} \leq \frac{2}{\alpha}, 
\quad \text{ for any $|y| \leq K$ }, 
\]
 hence $\mu = \frac{2 \alpha^2}{(\alpha+K)^3}$, and $L = \frac{2}{\alpha}$. 
\end{proof}

For example, one can choose $\alpha = K$ for which we have $\frac{L}{\mu} = 2$, which is slightly larger than one but is less sensitive to large, potentially outlying,  entries of $M$. Let us illustrate this on a simple example: 
\begin{align} 
M' 
%&  = \left( 
%\begin{array}{ccc} 
%   2  &  2   &  2 + \epsilon \\
%   0  &   1   &   0.5     \\
%    2  &   2   &  2  \\
%   1 &    2   &  1.5     \\ 
%   0 &   1   & 0.5  \end{array} \right) = WH + N \nonumber \\ 
& = 
\left( 
\begin{array}{cc} 
2 &   2\\ 
0  &   1\\ 
2   &  2\\ 
1    & 2\\ 
0    & 1 
\end{array} \right) 
\left( 
\begin{array}{ccc} 
1 & 0 & 0.5 \\ 
0 & 1 & 0.5  
\end{array} \right) 
+
\left( 
\begin{array}{ccc} 
   0  &  0   &  \epsilon \\
   0  &   0   &   0     \\
   0  &   0   &  0  \\
   0  &    0   &  0     \\ 
   0  &   0   & 0  \end{array} \right). \label{robex} 
\end{align} 
One can check that, for any $\epsilon \leq 0.69$, Algorithm~\ref{sepnmf} with $f(x) = ||x||_2^2$ recovers the first two columns of $M$, that is, the columns of $W$. However, using $f(x) = \sum_i \frac{x_i^2}{1 + |x_i|}$, Algorithm~\ref{sepnmf} recovers the columns of $W$ for any $\epsilon \leq 1.15$. 
Choosing appropriate function $f(x)$ depending on the input data matrix and the noise model is a topic for further research.

\begin{remark}[Relaxing Assumption~\ref{fass1}] \label{remRelf}
The condition that the gradient of $f$ must be Lipschitz continuous in Assumption~\ref{fass1} can be relaxed to the condition that the gradient of $f$ is continuously differentiable. In fact, in all our derivations, we have always assumed that $f$ was applied on a bounded set (more precisely, the ball $\{ x \ | \ ||x||_2 \leq K(W) \}$). Since $g \in \mathcal{C}^1$ implies that $g$ is locally Lipschitz continuous, the condition $f \in \mathcal{C}^2$ is sufficient for our analysis to hold.  Similarly, the strong convexity condition can be relaxed to local strong convexity. 

It would be interesting to investigate more general classes of functions for which our derivations hold. For example, for any increasing function $g : \mathbb{R}_+ \to \mathbb{R}$, the output of Algorithm~\ref{sepnmf} using $f$ or using $(g \circ f)$ will be the same, since $f(x) \geq f(y) \iff g\left(f(x)\right) \geq g\left(f(y)\right)$. Therefore, for our analysis to hold, it suffices that $(g \circ f)$ satisfies Assumption~\ref{fass1} for some increasing function $g$. For example, $||x||_2^4$ is not strongly convex although it will output the same result as $||x||_2^2$ hence will satisfy the same error bounds.  
\end{remark}

\subsection{Using $\ell_p$-norms} \label{pnorms}

%Another possible choice is $h(y) = |y|^p$ for $p > 1$, so that 
Another possible choice is 
\[
f(x) \; =  \; ||x||_p^2  \; =  \; \left( \sum_{i=1}^m |x_i|^p \right)^{2/p}. 
\]
For $1 < p \leq 2$, $f(x)$ is strongly convex with parameter $2(p-1)$ with respect to the norm $||.||_p$ \cite[Section 4.1.1]{JN08}, while its gradient is locally Lipschitz continuous (see Remark~\ref{remRelf}). For $2 \leq p < +\infty$, the gradient of $f(x)$ is Lipschitz continuous with respect to the norm $||.||_p$ with constant $2(p-1)$ (by duality), while it is locally strongly convex. Therefore, $f$ satisfies Assumption~\ref{fass1} for any $1 < p < +\infty$ in any bounded set, hence our analysis applies. 
Note that, for $p=1$ and $p = +\infty$, the algorithm is not guaranteed to work, even in the noiseless case (when points are on the boundary of the convex hull of the columns of $W$): consider for example the following separable matrices 
\[
M = \left( \begin{array}{ccc} 
1 & 0 & 0.5 \\ 
0 & 1 & 0.5  
\end{array} \right) 
= 
\left( \begin{array}{cc} 
1 & 0  \\ 
0 & 1  
\end{array} \right) 
\left( \begin{array}{ccc} 
1 & 0 & 0.5 \\ 
0 & 1 & 0.5  
\end{array} \right), 
\] 
for which the $\ell_1$-norm may fail (selecting the last column of $M$),  and 
\[
M = \left( \begin{array}{ccc} 
1 & 1 & 1 \\ 
0 & 1 & 0.5  
\end{array} \right) 
= 
\left( \begin{array}{cc} 
1 & 1  \\ 
0 & 1  
\end{array} \right) 
\left( \begin{array}{ccc} 
1 & 0 & 0.5 \\ 
0 & 1 & 0.5  
\end{array} \right), 
\] 
for which the $\ell_\infty$-norm may fail. Similarly as in the previous section, using $\ell_p$-norms with $p \neq 2$ might be rewarding in some cases. For example, for $1 < p < 2$, the $\ell_p$-norm is less sensitive to large entries of $M$. 
Consider the matrix from Equation~\eqref{robex}: for $p = 1.5$, Algorithm~\ref{sepnmf} extracts the columns of $W$ for any $\epsilon \leq 0.96$, while for $p = 4$, it only works for $\epsilon \leq 0.31$ (recall for $p = 2$ we had $\epsilon \leq 0.69$).

Algorithm~\ref{sepnmf} with $f(x) = ||x||_p$ has been previously introduced as the $\ell_p$-norm based pure pixel algorithm (TRI-P)~\cite{A11}, and shown to perform, in average, better than other existing techniques (namely, N-FINDR \cite{Win99}, VCA \cite{ND05}, and SGA \cite{CW06}). 
The authors actually only performed numerical experiments for $p=2$, but did not justify this choice (the reason possibly is that it gave the best numerical results, as our analysis suggests), 
%The authors proved that their algorithm works under Assumption~\ref{asssep}, that is, they proved Theorem~\ref{th1} in case $f(x) = ||x||_p^p$. 
and could not explain why Algorithm~\ref{sepnmf} performs better than other approaches in the noisy case.

\section{Numerical Experiments} \label{ne}

In the first part of this section, we compare Algorithm~\ref{sepnmf} with several fast hyperspectral unmixing algorithms under the linear mixing model and the pure-pixel assumption. 
%which are both cheap and designed to solve separable NMF problems. 
We first briefly describe them (computational cost and main properties) and then perform a series of experiments on synthetic data sets in order to highlight their properties. 
For comparisons of Algorithm~\ref{sepnmf} with other algorithms on other synthetic and real-world hyperspectral data sets, we refer the reader to \cite{MC01, RC03, CW07, ZR08, CM11, A11} since Algorithm~\ref{sepnmf} is a generalization of the algorithms proposed in \cite{MC01, RC03, CM11, A11}; see Section~\ref{choicef}.

In the second part of the section, we compare Algorithm~\ref{sepnmf} with the Algorithm of Bittorf et al.\@ \cite{BRRT12}.

\subsection{Comparison with Fast Hyperspectral Unmixing Algorithms} \label{ta}

%\subsubsection{Tested Algorithms} \label{ta}

We compare the following algorithms: 

\begin{enumerate}
\item \textbf{Algorithm~\ref{sepnmf} with $f(x) = ||x||_2^2$}. 
We will only test this variant because, according to our analysis, it is the most robust. (Comparing different variants of Algorithm~\ref{sepnmf} is a topic for further research.)  
%As mentioned in Section~\ref{algo1sec}, this algorithm is known as the automatic target generation process (ATGP) in the hyperspectral imaging literature.  
The computational cost is rather low: steps 3 and 5 are the only steps requiring computation, and have to be performed $r$ times. We have 
\begin{itemize}
\item Step 3. Compute the squared norm of the columns of $R$, which requires $n$ times $2m$ operations (squaring and summing the elements of each column), and extract the maximum, which requires $n$ comparisons, for a total of approximately $2mn$ operations. 
\item Step 5. It can be compute in the following way 
%\begin{align*}
\[
R  \leftarrow \left(I-\frac{u_j u_j^T}{||u_j||_2^2}\right)R   
 = R - \frac{u_j}{||u_j||_2^2}  (u_j^TR), 
%\end{align*}
\]
where computing $x^T = u_j^TR$ requires $2mn$ operations, $y = \frac{u_j}{||u_j||_2^2}$ $m$ operations, and $R - yx^T$ $2mn$ operations, for a total of approximately $4mn$ operations. 
\end{itemize}
The total computational cost of Algorithm~\ref{sepnmf} is then about $6mnr$ operations, plus some negligible terms.  

\begin{remark}[Sparse Matrices]
If the matrix $M$ is sparse, $R$ will eventually become dense which is often impractical. Therefore, $R$ should be kept in memory as the original matrix $M$ minus the rank-one updates. 
\end{remark}

\begin{remark}[Reducing Running Time]
Using recursively the formula 
\[
||(I-uu^T)v||_2^2 = ||v||_2^2 - (u^Tv)^2 , 
\]
for any  $u,v \in \mathbb{R}^m \text{ with } ||u||_2 = 1$, 
we can reduce the computational cost to {$2mnr + \mathcal{O}(mr^2)$ operations}. Although this formula is unstable when $u$ is almost parallel to $v$, this is negligible as we are only interested in the column with the largest norm (on all the tested data sets, including the 40000 randomly generated matrices below, we have always obtained the same results with both implementations). This is the version we have used for our numerical experiments as it turns out to be much faster (for example, on 200-by-200 matrices with $r=20$, it is about seven times faster, and on a real-world $188$-by-$47750$ hyperspectral image with $r=15$, about 20 times faster --taking less than half a second), and handles sparse matrices as it does not compute the residual matrix explicitly (for example, it takes about half a second for the  19949-by-43586 20-newsgroup data set for $r=20$). Note that all experiments have been performed with MATLAB R2011b on a laptop with 2GHz Intel Core i7-2630QM. 
The code is available at \url{https://sites.google.com/site/nicolasgillis/code}. 
\end{remark}

\item \textbf{Pure Pixel Index (PPI)} \cite{B94}.  PPI uses the fact that the maxima (and minima) of randomly generated linear functions over a polytope are attained on its vertices with probability one. Hence, PPI randomly generates a large number of linear functions (that is, functions $f(x) = c^T x$ where $c \in \mathbb{R}^{m}$ is uniformly distribution over the sphere), and identifies the columns of $M$ maximizing and minimizing these functions. 
Under the separability assumption, these columns must be, with probability one, vertices of the convex hull of the columns of $M$. Then, a score is attributed to each column of $M$: it is equal to the number of times the corresponding column is identified as a minimizer or maximizer of one of the randomly generated linear functions. Finally, the $r$ columns of $M$ with the largest score are identified as the columns of~$W$.  
Letting $K$ be the number of generated linear functions, we have to evaluate $K$ times linear functions overs $n$ vertices in dimension $m$ for a total computational cost of $\mathcal{O}(Kmn)$. For our experiments, we will use $K = 1000$. 
There are several pitfalls in using PPI: 
\begin{enumerate} 
\item It is not robust to noise. In fact, linear functions can be maximized at any vertex of the convex hull of a set of points. Therefore, in the noisy case, as soon as a column of the perturbed matrix $M'$ is not contained in the convex hull of the columns of $W$, it can be identified as a vertex. This can occur for arbitrarily small perturbation, as will be confirmed by the experiments below.

\item If not enough linear functions are generated, the algorithm might not be able to identify all the vertices (even in the noiseless case). This is particularly critical in case of ill-conditioning  because the probability that some vertices maximize a randomly generated linear function can be arbitrarily low. 

\item If the input noisy data matrix contains many columns close to a given column of matrix $W$, the score of these columns will be typically small (they essentially share the score of the original column of matrix $W$), while an isolated column which does not correspond to a column of $W$ could potentially have a higher score than these columns, hence be extracted. This can for example be rather critical for hyperspectral images where there typically are many pixels close to pure pixels 
(i.e., columns of $M$ corresponding to the same column of $W$). 
Moreover, for the same reasons, PPI might extract columns of $M$ corresponding to the same column of $W$.  
\end{enumerate}

\item \textbf{Vertex Component Analysis (VCA)} \cite{ND05}. The first step of VCA is to preprocess the data using principal component analysis which requires $O(nm^2 + m^3)$ \cite{ND05}. Then, the core of the algorithm requires $O(r m^2)$ operations (see below), for a total of $O(nm^2 + m^3)$ operations\footnote{The code is available at \url{http://www.lx.it.pt/~bioucas/code.htm}.}.  
Notice that the preprocessing is particularly well suited for data sets where $m \ll n$, such as hyperspectral images, where $m$ is the number of hyperspectral images with $m \sim 100$,  while $n$ is the number of pixels per image with $n \sim 10^6$.  
The core of VCA is very similar to Algorithm~\ref{sepnmf}:  
%as it is also a recursive algorithm, and, 
at each step, it projects the data onto the orthogonal complement of the extracted column. However, instead of using a strictly convex function to identify a vertex of the convex hull of $M$ (as in Algorithm~\ref{sepnmf}), it uses a randomly generated linear function  (that is, it selects the column maximizing the function $f(x) = c^T x$ where $c$ is randomly generated, as PPI does). Therefore, for the same reasons as for PPI, the algorithm is not robust. However, it solves the second and third pitfalls of PPI (see point (b) and (c) above), that is, it will always be able to identify enough vertices, and a cluster of points around a vertex are more likely to be extracted than an isolated point. 
Note that, in the VCA implementation, only one linear function is generated at each step which makes it rather sensitive to this choice.  
Finally, the two main differences between VCA and Algorithm~\ref{sepnmf} are  that: (1) VCA uses a pre-processing (although we could implement a version of Algorithm~\ref{sepnmf} with the same pre-processing), and (2) VCA uses randomly generated linear functions to pick vertices of the convex hull of the columns of $M$ (which makes it non-robust to noise, and also non-deterministic).

\item \textbf{Simplex Volume Maximization (SiVM)} \cite{T12}. SiVM recursively extracts columns of the matrix $M$ while trying to maximize the volume of the convex hull of the corresponding columns. %(the first step being almost identical as the one of Algorithm~\ref{sepnmf}). 
Because evaluating the volumes induced by adding a column not yet extracted to the previously selected ones is computationally expensive, the function is approximated via a heuristic, for a total computational cost of $O(mnr)$ operations (as for Algorithm~\ref{sepnmf}). The heuristic assumes that the columns of~$W$ are located at the same distance, that is, $||w_i-w_j||_2 = ||w_k-w_l||_2$ for all $i \neq j$ and $k \neq l$. Therefore, there is not guarantee that the algorithm will work, even in the noiseless case; this will be particularly critical for ill-conditioned problems, which will be confirmed by the experiments below.

\end{enumerate}

We now generate several synthetic data sets allowing to highlight the properties of the different algorithms, and in particular of Algorithm~\ref{sepnmf}. 
We are going to consider noisy separable matrices generated as follows.  
\begin{enumerate}

\item The matrix $W$ will be generated in two different ways : 
	\begin{enumerate}
	
		\item[(a)] \emph{Uniform Distribution.} 
		The entries of matrix $W \in \mathbb{R}^{200 \times 20}$ are randomly and independently generated following an uniform distribution between 0 and 1 (using the \emph{rand()} function of MATLAB). 
		
		\item[(b)] \emph{Ill-conditioned.} 
		First, we generate a matrix $W' \in \mathbb{R}^{200 \times 20}$ as above (that is, using the \emph{rand()} function of MATLAB). Then, we compute the compact SVD $(U \in \mathbb{R}^{200 \times 20},  \Sigma \in \mathbb{R}^{20 \times 20}, V \in \mathbb{R}^{20 \times 20})$  of $W' = U\Sigma V^T$, and finally generate $W = U S V^T$ where $S$ is a diagonal matrix whose diagonal entries are equal to $\alpha^{i-1}$ for  $i=1,2,\dots,20$ where $\alpha^{19} = 10^{-3}$ (that is, $\alpha = 0.695$) in such a way that $\sigma_1(W) = 1$ and $\sigma_{20}(W) = 10^{-3}$, hence $\kappa(W) = 1000$.  (Note that $W$ might not be nonnegative, a situation which is handled by the different algorithms.)  
	\end{enumerate}		
	
\item The matrices $H$ and $N$ will be generated in two different ways as well : %(leading to a total of four experiments) 

	\begin{enumerate}
	
		\item[(c)] \emph{Middle Points.} 
	We set $H = [I_{20}, \, H'] \in \mathbb{R}^{20 \times 210}$, where the columns of $H'$ contain all possible combinations of two non-zero entries equal to 0.5 at different positions (hence $H'$ has $\binom{20}{2} = 190$ columns). This means that the first 20 columns of $M$ are equal to the columns of $W$ while the 190 remaining ones are equal to the middle points of the columns of $W$.  We do not perturb the first 20 columns of $M$ (that is, $n_i = 0$ for $1 \leq i \leq 20$), while, for the 190 remaining ones, we use 
\[
n_i = \delta \, (m_i-\bar{w}) \; \text{ for } 21 \leq i \leq 210, \quad \delta \geq 0, 
\]
where $\bar{w}$ is the average of the columns of $W$ (geometrically, this is the vertex centroid of the convex hull of the columns of $W$). This means that we move the columns of $M$ toward the outside of the convex hull of the columns of $W$. Hence, for any $\delta > 0$, the columns of $M'$ are not contained in the convex hull of the columns of $W$ (although the rank of $M'$ remains equal to 20). 

		\item[(d)] \emph{Dirichlet and Gaussian Distributions.} We set $H = [I_{20}, \, I_{20}, \, H'] \in \mathbb{R}^{20 \times 240}$, where the columns of $H'$ are generated following a Dirichlet distribution whose $r$ parameters are chosen uniformly in $[0,1]$ (the Dirichlet distribution generates vectors $h'_i$ on the boundary of $\Delta^r$, that is, $\sum_k h'_i(k) = 1$ $\forall i$). 
				We perturb each entry of $M$ independently using the normal distribution: 
\[
n_i(k) \sim \delta \, \mathcal{N}(0,1) \,\text{ for } 1 \leq i \leq 240, 1 \leq k \leq 200. 
\] 
(The expected value of the $\ell_2$-norm of the columns of $N$ is $\sqrt{m} \delta$, the square root of the expected value of a Chi-squared distribution.) Notice that each column of $W$ is present twice as a column of $M$ (in terms of hyperspectral unmixing, this means that there are two pure pixels per endmember). 

	\end{enumerate}	
	\end{enumerate}
	
Finally, we construct the noisy separable matrix $M' = WH + N$ in four different ways, see Table~\ref{experi}, where $W$, $H$ and $N$ are generated as described above for a total of four experiments. 
	\begin{table}[ht!]
\begin{center}
\caption{Generation of the noisy separable matrices for the different experiments and average value of $\kappa(W)$, $K(W)$, and $\sigma_r(W)$.} 
\begin{tabular}{|c|c|c|c|c|}
\hline
					& Exp. 1   &   Exp. 2 &  Exp. 3 & Exp. 4  \\ 
\hline 
$W$		&	 (a)  & (a)  	 & 	(b)	&  (b) \\ 
$N$ and $H$  			 &	 (c)   &  (d)  & 	(c) 	&  (d) \\ \hline 
$\kappa(W)$  			 &	10.84  & 10.84 & 	1000 	  &  1000 \\
%$\omega(W)$  			 &	5.11   &  5.12 & 	0.096 	& 0.096 \\
$K(W)$  			     &	8.64   &  8.64 & 	0.41 	&  0.41 \\
$\sigma_r(W)$  	   &	2.95   &  2.95 & $10^{-3}$	& $10^{-3}$  \\
Average ($\max_i ||n_i||_2$ $\delta^{-1}$) &	3.05  &  16.15  & 0.29 	&  16.15     \\ 
\hline  
\end{tabular}
\label{experi}
\end{center}
\end{table}
%\subsection{Experiments} 
For each experiment, % and for each value of $\delta$ (100 values equally spaced), 
we generate 100 matrices for 100 different values of $\delta$ and compute the percentage of columns of $W$ that the algorithms were able to identify (hence the higher the curve, the better); see Figure~\ref{expdel}. (Note that we then have 10000 matrices generated for each experiment.)  
\begin{figure}[ht!]
\begin{center}
\includegraphics[width=8cm]{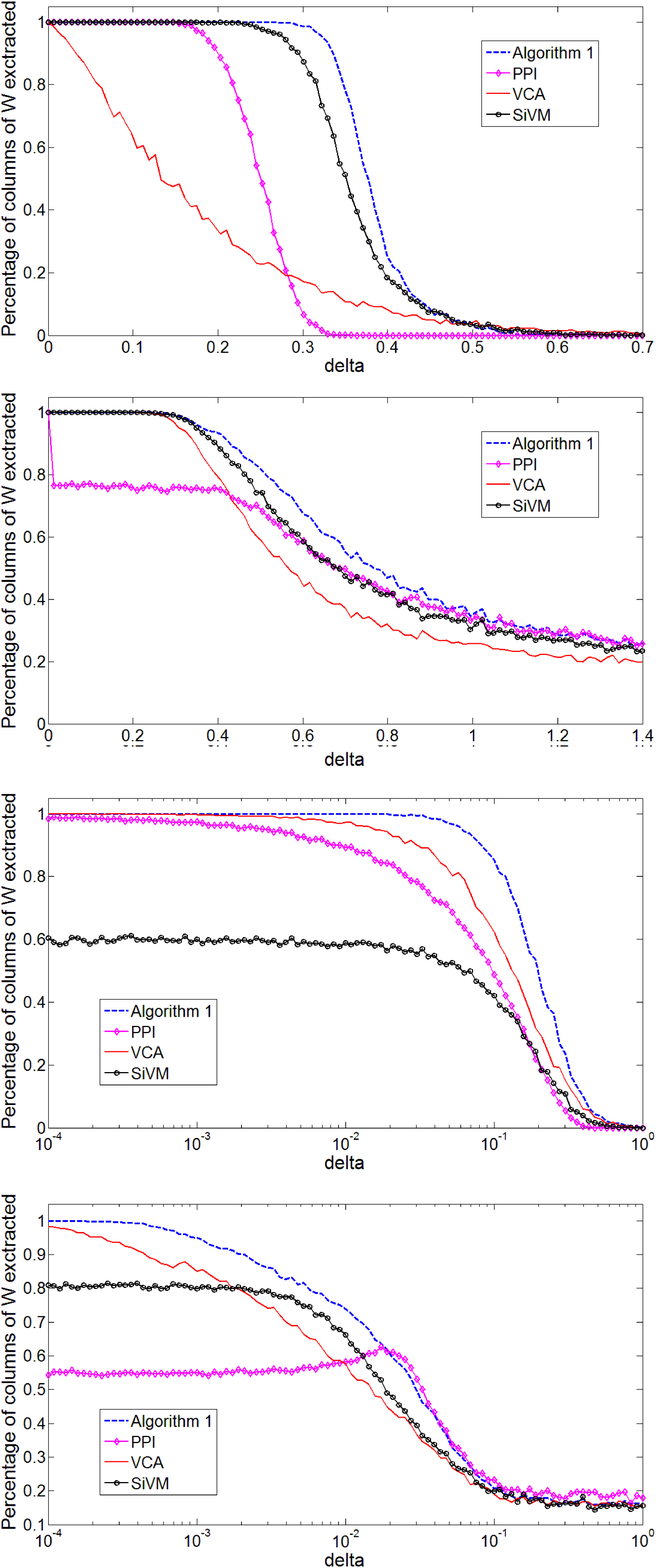}
\caption{Comparison of Algorithm~\ref{sepnmf}, PPI, VCA and SiVM. From top to bottom: Exp.\@~1, Exp.\@~2, Exp.\@~3, and Exp.\@~4.} 
\label{expdel}
\end{center}
\end{figure} 
We observe the following 
\begin{itemize}
\item Algorithm~\ref{sepnmf} is the most robust algorithm as it is able to identify all the columns of $W$ for the largest values of the perturbation $\delta$ for all experiments; see Table~\ref{robustn} and Figure~\ref{expdel}. 
\begin{table}[ht!]
\begin{center}
\caption{Robustness: maximum values of $\delta$ for perfect recovery.} 
\begin{tabular}{|c||c|c|c|c|}
\hline
					& Exp. 1   &   Exp. 2 &  Exp. 3 & Exp. 4  \\ 
\hline \hline
Algorithm 1	 &	0.252   &  0.238  	 & 	0.011	& $1.74^{\ast} 10^{-4}$ \\ 
PPI        &	 0.175   & 0 & 	0 	&  0 \\
VCA       &	  0 & 0.210 &  0	 	&  0 \\
SiVM      &	0.126  & 0.224 & 	/ 	&   / \\ 
\hline  
\end{tabular}
\label{robustn}

(The sign / means that the algorithm failed even in the noiseless case.) 
\end{center}
\end{table}

\item  In Exp.\@~1, PPI and SiVM perform relatively well, the reason being that the matrix $W$ is well-conditioned (see Table~\ref{experi}) while VCA is not robust to any noise. In fact, as explained in Section~\ref{ta}, VCA only uses one randomly generated linear function to identify a column of $W$ at each step, hence can potentially extract any column of $M$ since they all are vertices of the convex hull of the columns of $M$ (the last columns of $M$ are the middle points of the columns of $W$ and are perturbed toward the outside of the convex hull of the columns of~$W$).   

\item  In Exp.\@~2, the matrix $W$ is well-conditioned so that SiVM still performs well.  PPI is now unable to identify all columns of $M$, because of the repetition in the data set (each column of $W$ is present twice as a column of $M$)\footnote{For $\delta = 0$, we observe that our implementation of the PPI algorithm actually recovers all columns of $W$. The reason is that the first columns of $M$ are exactly equal to each other and that the MATLAB  \emph{max(.)} function only outputs the smallest index corresponding to a maximum value. This is why PPI works in the noiseless case even when there are duplicates.}. VCA now performs much better because the columns of $M$ are strictly contained in the interior of the convex hull of the columns of~$W$.

\item  In Exp.\@~3 and 4, SiVM performs very poorly because of the ill-conditioning of matrix $W$. 

\item In Exp.\@~3, as opposed to Exp.\@ 1, PPI is no longer robust because of ill-conditioning, although more than 97\% of the columns of $W$ are perfectly extracted for all $\delta \leq 10^{-3}$. VCA is not robust but performs better than PPI, and extracts more than 97\% of the columns of $W$ for all $\delta \leq 10^{-2}$ (note that Algorithm~\ref{sepnmf} does for $\delta \leq 0.05$).

\item In Exp.\@~4, PPI is not able to identify all the columns of $W$ because of the repetition, while, as opposed to Exp.\@~2,  VCA is not robust to any noise because of ill-conditioning.

\item Algorithm~\ref{sepnmf} is the fastest algorithm although PPI and SiVM have roughly the same computational time. VCA is slower as it uses PCA as a preprocessing; see Table~\ref{ct}. 
\begin{table}[ht!]
\begin{center}
\caption{Average computational time (s.) for the different algorithms.} 
\begin{tabular}{|c||c|c|c|c|}
\hline
					& Exp. 1   &   Exp. 2 &  Exp. 3 & Exp. 4  \\ 
\hline \hline
Algorithm 1		&	 0.0080  & 0.0087  	 & 	 0.0086	&  0.0086 \\ 
PPI   			 &	 0.052  & 0.051 & 	0.049 	&  0.054 \\
VCA   				 &	 2.69  & 0.24 & 	2.71 	& 1.41 \\ 
SiVM   			 & 0.025	   & 0.027 & 	0.028 	&  0.027 \\
\hline  
\end{tabular}
\label{ct}
\end{center}
\end{table}

\end{itemize}

These experiments also show that the error bound derived in Theorem~\ref{threc} is rather loose, which can be partly explained by the fact that our analysis considers the worst-case scenario (while our experiments use either a structured noise or Gaussian noise). 
Recall that the value of $\epsilon$ in Theorem~\ref{threc} is the smallest value such that $||n_i||_2 \leq \epsilon$ for all $i$; see Equation~\eqref{sepnmfbound}.  
Table~\ref{experi} gives the average value of the maximum norm of the columns of $N$ for each experiment. 
%For example, for Exp.\@~1, it is slightly larger to $\sqrt{m} \delta$ (each entry of $N$ is generated according to a normal distribution) so that, in average,  $\sqrt{m} \delta \approx \epsilon$. 
Based on these values, the first row of Table~\ref{errb} shows the average upper  bound for $\delta$ to guarantee recovery; see Theorem~\ref{threc}.
%(we used the improved bound from Remark~\ref{imprbnd} when $\mu = L$, that is, the constant 80 can be replaced with 64). 
\begin{table}[ht!]
\begin{center}
\caption{Comparison of the average value of $\delta$ predicted by Theorem~\ref{threc}  %and \ref{apost} 
to guarantee recovery %for Algorithm~\ref{sepnmf} with $f(x) = ||x||_2^2$, 
compared to the observed values.} 
\begin{tabular}{|c||c|c|c|c|}
\hline
					& Exp. 1   &   Exp. 2 &  Exp. 3 & Exp. 4  \\ 
\hline   %\hline  
Th.~\ref{threc} 
& $\left. 3.7^{\ast}10^{-5} \right.$ & $7^{\ast}10^{-6}$ 	 & $6.7^{\ast}10^{-12}$   &  $1.2^{\ast}10^{-13}$ \\ 
%Th.~\ref{apost}, Equation~\eqref{Th3a} 
%& 0.053  & 0.070  & 5.97\, $10^{-7}$ 	  & 7.08 \, $10^{-7}$  \\
%Th.~\ref{apost}, Equation~\eqref{Th3b} 
%& 9\,	$10^{-3}$ & 0.012  & 7.26\, $10^{-10}$ 	  & 8.6 \, $10^{-10}$  \\
Observed   &	0.259   &  0.238   & 	 0.011   	&  $1.74^{\ast}10^{-4}$ \\ 
\hline  
\end{tabular}
\label{errb}
\end{center}
\end{table}

\subsection{Comparison with the Algorithm of Bittorf et al.\@ \cite{BRRT12}} \label{compaB}

In this section, we compare the Algorithm of Bittorf et al.\@  (BRRT) (Algorithm 3 in {\cite{BRRT12}; see also Algorithm~2 in~\cite{G13})\footnote{We do not perform a comparison with the algorithm of Arora et al.\@ \cite{AGKM11} as it is not very practical (the value of $\alpha$ has to be estimated, see Section~\ref{compa}) and has already been shown to perform similarly as BRRT in \cite{BRRT12}.} with Algorithm~\ref{sepnmf}. 
BRRT has to solve a linear program with $\mathcal{O}(n^2)$ variables which we solve using CVX~\cite{cvx}. Therefore we are only able to solve small-scale problems (in fact, CVX uses an interior-point method): we perform {exactly} the same experiments as in the previous section but for $m=10$ and $r=5$ for all experiments, so that  
\begin{itemize}
\item $n = 5 + \binom{5}{2} = 15$ for the first and third experiments (the last ten columns of $M$ are the middle points of the five columns of $W$). 
\item $n = 5 + 5 + 10 = 20$ for the second and fourth experiments (we repeat twice each endmember, and add 10 points in the convex hull of the columns of $W$). 
\end{itemize}
%Note that, as the noise matrix $N$ is known, the parameter $\epsilon_1$ can be estimated exactly. 
The average running time for BRRT on these data sets using CVX \cite{cvx} is about two seconds while, for Algorithm~\ref{sepnmf}, it is less than $10^{-3}$ seconds. (Bittorf et al.~\cite{BRRT12} propose a more efficient solver than CVX for their LP instances.  As mentioned in Section~\ref{compa}, even with their more efficient solver, Algorithm~\ref{sepnmf} is much faster for large $n$.) Figure~\ref{expdelb} shows the percentage of correctly extracted columns with respect to~$\delta$, while Table~\ref{robustn2} shows the robustness of both algorithms. 
\begin{table}[ht!]
\begin{center}
\caption{Robustness: maximum values of $\delta$ for perfect recovery.} 
\begin{tabular}{|c||c|c|c|c|}
\hline
					& Exp. 1   &   Exp. 2 &  Exp. 3 & Exp. 4  \\ 
\hline \hline
Algorithm~\ref{sepnmf}	 & \textbf{0.070}  &  $1.4^{\ast}10^{-6}$   	 &  $7.9^{\ast}10^{-4}$	  	&  $\mathbf{2.4^{\ast}10^{-6}}$ \\ 
BRRT \cite{BRRT12}  & 	0.042	   &  $\mathbf{5.5^{\ast}10^{-6}}$   & $\mathbf{1.8^{\ast}10^{-3}}$  	&   $10^{-6}$ \\\hline  
\end{tabular}
\label{robustn2}
\end{center}
\end{table}
\begin{figure}[ht!]
\begin{center}
\includegraphics[width=7.5cm]{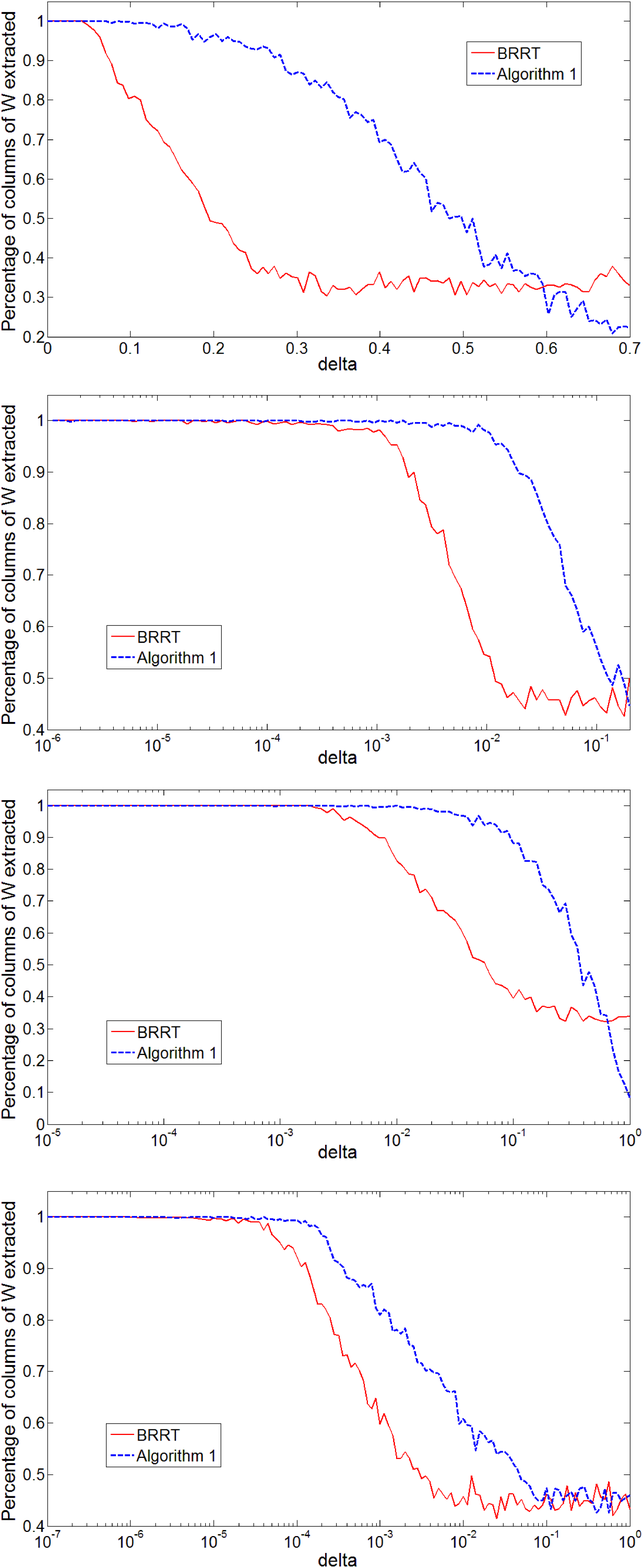}
\caption{Comparison of Algorithm~\ref{sepnmf} and BRRT \cite{BRRT12}. From top to bottom: Exp.\@~1, Exp.\@~2, Exp.\@~3, and Exp.\@~4.} 
\label{expdelb}
\end{center}
\end{figure} 

Quite surprisingly, Algorithm~\ref{sepnmf} performs in average better than BRRT. Although BRRT is more robust in two of the four experiments (that is, it extracts correctly all columns of $W$ for a larger value of $\delta$), the percentage of columns it is able to correctly extract decreases much faster as the noise level increases. For example, Table~\ref{robustn3} shows the maximum value of $\delta$ for which 99\% percent of the columns of $W$ are correctly extracted. In that case, Algorithm~\ref{sepnmf} always performs better. 
\begin{table}[ht!]
\begin{center}
\caption{Maximum values of $\delta$ for recovering 99\% of the columns of $W$.} 
\begin{tabular}{|c||c|c|c|c|}
\hline
					& Exp. 1   &   Exp. 2 &  Exp. 3 & Exp. 4  \\ 
\hline \hline
Algorithm~\ref{sepnmf}	& 0.126  &  	$2.8^{\ast}10^{-3}$ &  $1.2^{\ast}10^{-2}$ 	  	&  $10^{-4}$ \\ 
BRRT \cite{BRRT12}  & 0.05	   & $4.0^{\ast}10^{-4}$   & $2.2^{\ast}10^{-3}$  	&  $1.8^{\ast}10^{-5}$   \\\hline  
\end{tabular}
\label{robustn3}
\end{center}
\end{table}
A possible explanation for this behavior is that, when the noise is too large, the condition for recovery are not satisfied as the input matrix is far from being separable.  However, using Algorithm~\ref{sepnmf} still makes sense as it extracts columns whose convex hull has large volume \cite{AM09,AM10} while it is not clear what BRRT does in that situation (as it heavily relies on the separability assumption). Therefore, although BRRT guarantees perfect recovery for higher noise levels, it appears that, in practice, when the noise level is high, Algorithm~\ref{sepnmf} is preferable.

\section{Conclusion and Further Work}

In this paper, we have introduced and analyzed a new family of fast and robust recursive algorithms for separable NMF problems which are equivalent to hyperspectral unmixing problems under the linear mixing model and the pure-pixel assumption. This family generalizes several existing hyperspectral unmixing algorithms, and our analysis provides a theoretical framework to explain the better performances of these approaches. 
In particular, our analysis explains why algorithms like PPI and VCA are less robust against noise compared to Algorithm 1. %or other approaches introduced in \cite{AGKM11, BRRT12}. 

Many questions remain open, and would be interesting directions for further research: 
\begin{itemize}

\item Is it possible to provide  better error bounds for Algorithm~\ref{sepnmf} than the ones of Theorem~\ref{threc}? In other words, is our analysis tight? Also, can we improve the bounds if we assume specific generative and/or noise models? 

\item How can we choose appropriate functions $f(x)$ for Algorithm~\ref{sepnmf} depending on the input data matrix? 

\item Can we design other robust and fast algorithms for the separable NMF problem leading to better error bounds? %In particular, it would be interesting
%it would be interesting to find out whether Algorithm~\ref{sepnmf} is `optimal', that is, does there exist an algorithm running in $\mathcal{O}(mnr)$  guaranteeing better error bounds? 

\end{itemize}

 \section*{Acknowledgments}
 
  The authors would like to thank the reviewers for their feedback which helped improve the paper significantly.

%**********************************************************************************************************************************************

\bibliographystyle{spmpsci}  
\bibliography{fastNMF}

\end{document}